\documentclass{article}

\usepackage{microtype}
\usepackage{graphicx}
\usepackage{subfigure}
\usepackage{booktabs} 

\usepackage{hyperref}

\usepackage[accepted]{icml2024}

\usepackage{amsmath}
\usepackage{amssymb}
\usepackage{mathtools}
\usepackage{amsthm}

\usepackage[capitalize,noabbrev]{cleveref}

\theoremstyle{plain}
\newtheorem{theorem}{Theorem}[section]
\newtheorem{proposition}[theorem]{Proposition}
\newtheorem{lemma}[theorem]{Lemma}
\newtheorem{corollary}[theorem]{Corollary}
\theoremstyle{definition}
\newtheorem{definition}[theorem]{Definition}
\newtheorem{assumption}[theorem]{Assumption}
\theoremstyle{remark}

\usepackage[textsize=tiny]{todonotes}

\newcommand{\R}{\mathbb{R}}
\newcommand{\eps}{\varepsilon}
\newcommand{\Dc}{\mathcal D}
\newcommand{\Kc}{\mathcal K}
\newcommand{\Bc}{\mathcal B}
\newcommand{\Lc}{\mathcal L}
\newcommand{\Wstar}{\mathcal W_\star}
\newcommand{\F}{\mathcal F}
\newcommand{\HL}{H_{L} }
\newcommand{\epsLoc}{\varepsilon_{\textup{loc}}}
\newcommand{\wloc}{w_{\textup{loc}}}

\newcommand{\Neps}{\mathcal N_{\epsLoc}(w_\star)}
\newcommand{\N}{\mathcal N}
\newcommand{\dist}{\textup{dist}}
\newcommand{\PL}{P\L$^{\star}$}
\newcommand{\A}{\mathcal A}
\newcommand{\lamMin}{\lambda_{\textup{min}}}
\newcommand{\lamMax}{\lambda_{\textup{max}}}
\newcommand{\nres}{n_\textup{res}}
\newcommand{\nbc}{n_{\textup{bc}}}
\newcommand{\lbfgs}{L-BFGS}
\newcommand{\al}{Adam+\lbfgs}
\newcommand{\aln}{Adam+\lbfgs+NNCG}
\newcommand{\alg}{Adam+\lbfgs+GD}

\usepackage{eqparbox}
\renewcommand{\algorithmiccomment}[1]{\hfill \(\triangleright\) #1}

\usepackage{multirow}

\newcommand{\pnote}[1]{}
\renewcommand{\pnote}[1]{\textcolor{red}{\textbf{[PR: #1]}}}

\icmltitlerunning{Challenges in Training PINNs}

\begin{document}

\twocolumn[
\icmltitle{Challenges in Training PINNs: A Loss Landscape Perspective}



\icmlsetsymbol{equal}{*}

\begin{icmlauthorlist}
\icmlauthor{Pratik Rathore}{stanfordee}
\icmlauthor{Weimu Lei}{stanfordicme}
\icmlauthor{Zachary Frangella}{stanfordmse}
\icmlauthor{Lu Lu}{yalestat}
\icmlauthor{Madeleine Udell}{stanfordicme,stanfordmse}
\end{icmlauthorlist}

\icmlaffiliation{stanfordee}{Department of Electrical Engineering, Stanford University, Stanford, CA, USA}
\icmlaffiliation{stanfordicme}{ICME, Stanford University, Stanford, CA, USA}
\icmlaffiliation{stanfordmse}{Department of Management Science \& Engineering, Stanford University, Stanford, CA, USA}
\icmlaffiliation{yalestat}{Department of Statistics and Data Science, Yale University, New Haven, CT, USA}

\icmlcorrespondingauthor{Pratik Rathore}{pratikr@stanford.edu}

\icmlkeywords{Physics-informed neural networks, scientific machine learning, loss landscape, optimization, preconditioning}

\vskip 0.3in
]



\printAffiliationsAndNotice{}  

\begin{abstract}
This paper explores challenges in training Physics-Informed Neural Networks (PINNs), 
emphasizing the role of the loss landscape in the training process. 
We examine difficulties in minimizing the PINN loss function, particularly due to ill-conditioning caused by differential operators in the residual term. 
We compare gradient-based optimizers Adam, L-BFGS, and their combination \al{}, showing the superiority of \al{}, and introduce a novel second-order optimizer, NysNewton-CG (NNCG), which significantly improves PINN performance. 
Theoretically, our work elucidates the connection between ill-conditioned differential operators and 
ill-conditioning in the PINN loss and shows the benefits of combining first- and second-order optimization methods. 
Our work presents valuable insights and more powerful optimization strategies for training PINNs, 
which could improve the utility of PINNs for solving difficult partial differential equations.
\end{abstract}

\section{Introduction}
The study of Partial Differential Equations (PDEs) grounds a wide variety of scientific and engineering fields,
yet these fundamental physical equations are often difficult to solve numerically.
Recently, neural network-based approaches including physics-informed neural networks (PINNs) 
have shown promise to solve both forward and inverse problems involving PDEs \cite{raissi2019unified,e2018deep,lu2021deeponet,lu2021deepxde,karniadakis2021physicsinformed,cuomo2022scientific}.
PINNs parameterize the solution to a PDE with a neural network, and are often fit by minimizing a least-squares loss involving the PDE residual, boundary condition(s), and initial condition(s).
The promise of PINNs is the potential to obtain solutions to PDEs without discretizing or meshing the space,
enabling scalable solutions to high-dimensional problems that currently require weeks on advanced supercomputers.
This loss is typically minimized with gradient-based optimizers such as Adam \cite{kingma2014adam}, L-BFGS \cite{liu1989limited}, or a combination of both.

However, 
the challenge of optimizing PINNs restricts the application and development of these methods.
Previous work has shown that the PINN loss is difficult to minimize \cite{krishnapriyan2021characterizing,wang2021understanding,wang2022when,de2023operator} even in simple settings. 
As a result, the PINN often fails to learn the solution.
Furthermore, optimization challenges can obscure the effectiveness of new neural network architectures for PINNs, 
as an apparently inferior performance may stem from insufficient loss function optimization 
rather than inherent limitations of an architecture.
A simple, reliable training paradigm is critical to enable wider adoption of PINNs.

\begin{figure}
    \centering
    \includegraphics[scale=0.6]{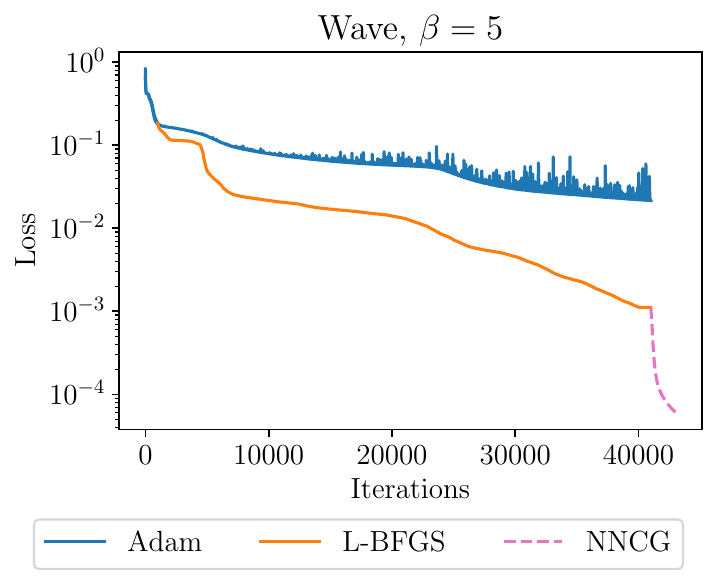}
    \caption{On the wave PDE, Adam converges slowly due to ill-conditioning and the combined \al{} optimizer stalls after about 40000 steps. Running NNCG (our method) after \al{} provides further improvement.}
    \label{fig:under_optimization_intro}
\end{figure}

This work explores the loss landscape of PINNs and the challenges this landscape poses for gradient-based optimization methods.
We provide insights from optimization theory that explain slow convergence of first-order methods such as Adam and show how ill-conditioned differential operators make optimization difficult.
We also use our theoretical insights to improve the PINN training pipeline by combining existing and new optimization methods.

The most closely related works to ours are \citet{krishnapriyan2021characterizing,de2023operator}, 
which both identify ill-conditioning in the PINN loss.
Unlike \citet{krishnapriyan2021characterizing}, we empirically confirm the ill-conditioning of the loss 
by visualizing the spectrum of the Hessian and demonstrating how quasi-Newton methods improve the conditioning.
Our theoretical results directly show how an ill-conditioned linear operator 
induces an ill-conditioned objective, 
in contrast to the approach in \citet{de2023operator} which relies on a linearization.


\textbf{Contributions.} We highlight contributions of this paper:
\begin{itemize}
\setlength\itemsep{-1pt}
    \item We demonstrate that the loss landscape of PINNs is ill-conditioned due to differential operators in the residual term and show that quasi-Newton methods improve the conditioning by 1000$\times$ or more (\cref{sec:loss_landscape}).
    \item We compare three optimizers frequently used for training PINNs: (i) Adam, (ii) \lbfgs{}, and (iii) Adam followed by L-BFGS (referred to as \al). 
    We show that \al{} is superior across a variety of network sizes (\cref{sec:opt_comparison}).
    

    \item We show the PINN solution
    resembles the true PDE solution only for extremely small loss values (\cref{sec:near_zero_loss}).
    However, we find that the loss returned by \al{} can be improved further, which also improves the PINN solution (\cref{sec:under_optimized}).

    \item Motivated by the ill-conditioned loss landscape, we introduce a novel second-order optimizer, NysNewton-CG (NNCG). 
    We show NNCG  can significantly improve the solution returned by \al{} (\cref{fig:under_optimization_intro,sec:under_optimized}). 
    \item We prove that ill-conditioned differential operators lead to an ill-conditioned PINN loss (\cref{sec:theory}). We also prove that combining first- and second-order methods (e.g., \al) leads to fast convergence, providing justification for the importance of the combined method (\cref{sec:theory}).
\end{itemize}

\textbf{Notation.}
We denote the Euclidean norm by $\|\cdot\|_2$ and use $\|M\|$ to denote the operator norm of $M\in \R^{m\times n}$.
For a smooth function $f: \R^p\rightarrow \R$, we denote its gradient at $w \in \R^p$ by $\nabla f(w)$ and its Hessian by $H_f(w)$.
We write $\partial_{w_i}f$ for $\partial f/\partial w_i$.
For $\Omega \subset \R^d$, we denote its boundary by $\partial \Omega$.
For any $m\in \mathbb{N}$, we use $I_m$ to denote the $m\times m$ identity matrix. 
Finally, we use $\preceq$ to denote the Loewner ordering on the convex cone of positive semidefinite matrices.
\section{Problem Setup}
This section introduces physics-informed neural networks as optimization problems and our experimental methodology.

\subsection{Physics-informed Neural Networks}
The goal of physics-informed neural networks is to solve partial differential equations.
Similar to prior work \cite{lu2021deepxde,hao2023pinnacle}, we consider the following system of partial differential equations:
\begin{subequations}
    \label{eq:gen_pde}
    \begin{align}
        & \Dc[u(x),x] = 0, \quad x\in \Omega, \\
        & \Bc[u(x),x] = 0, \quad x\in \partial \Omega, 
    \end{align} 
\end{subequations}

where $\Dc$ is a differential operator defining the PDE, $\Bc$ is an operator associated with the boundary and/or initial conditions, and $\Omega \subseteq \R^d$. 
To solve \eqref{eq:gen_pde}, PINNs model $u$ as a neural network $u(x; w)$ (often a multi-layer perceptron (MLP)) and approximate the true solution by the network whose weights solve the following non-linear least-squares problem:
\begin{align}
\label{eq:pinn_prob_gen}
    \underset{w\in \R^p}{\mbox{minimize}}~L(w) \coloneqq & \frac{1}{2\nres}\sum_{i=1}^{\nres}\left(\Dc[u(x_r^i; w),x_r^i]\right)^2\\ 
    &+\frac{1}{2\nbc}\sum^{\nbc}_{i=1}\left(\Bc[u(x_b^j;w),x_b^j]\right)^2. \nonumber
\end{align}
Here $\{x_r^i\}^{\nres}_{i=1}$ are the residual points and $\{x^j_b\}^{\nbc}_{j=1}$ are the boundary/initial points.
The first loss term measures how much $u(x;w)$ fails to satisfy the PDE, while the second term measures how much $u(x;w)$ fails to satisfy the boundary/initial conditions.

For this loss, $L(w)=0$ means that $u(x;w)$ exactly satisfies the PDE and boundary/initial conditions at the training points.
In deep learning, this condition is called \emph{interpolation} \cite{zhang2021understanding,belkin2021fit}.
There is no noise in \eqref{eq:gen_pde},
so the true solution of the PDE would make \eqref{eq:pinn_prob_gen} equal to zero. 
Hence a PINN approach should choose an architecture and an optimizer to achieve interpolation.
Moreover, smaller training error corresponds to better generalization for PINNs \cite{mishra2023estimates}.
Common optimizers for \eqref{eq:pinn_prob_gen} include Adam, L-BFGS, and Adam+L-BFGS \cite{raissi2019unified,krishnapriyan2021characterizing,hao2023pinnacle}.


\subsection{Experimental Methodology}
We conduct experiments on optimizing PINNs for convection, wave PDEs, and a reaction ODE. 
These equations have been studied in previous works investigating difficulties in training PINNs; we use the formulations in \citet{krishnapriyan2021characterizing, wang2022when} for our experiments. 
The coefficient settings we use for these equations are considered challenging in the literature \cite{krishnapriyan2021characterizing, wang2022when}.
\cref{sec:problem_setup_additional} contains additional details.

We compare the performance of Adam, \lbfgs{}, and \al{} on training PINNs for all three classes of PDEs. 
For Adam, we tune the learning rate by a grid search on $\{10^{-5}, 10^{-4}, 10^{-3}, 10^{-2}, 10^{-1}\}$.
For \lbfgs, we use the default learning rate $1.0$, memory size $100$, and strong Wolfe line search.
For \al, we tune the learning rate for Adam as before, and also vary the switch from Adam to \lbfgs{} (after 1000, 11000, 31000 iterations).
These correspond to \al{} (1k), \al{} (11k), and \al{} (31k) in our figures.
All three methods are run for a total of 41000 iterations.

We use multilayer perceptrons (MLPs) with tanh activations and three hidden layers. These MLPs have widths 50, 100, 200, or 400.
We initialize these networks with the Xavier normal initialization \cite{glorot2010understanding} and all biases equal to zero.
Each combination of PDE, optimizer, and MLP architecture is run with 5 random seeds.

We use 10000 residual points randomly sampled from a $255 \times 100$ grid on the interior of the problem domain. 
We use 257 equally spaced points for the initial conditions and 101 equally spaced points for each boundary condition.

We assess the discrepancy between the PINN solution and the ground truth using $\ell_2$ relative error (L2RE), a standard metric in the PINN literature. Let $y = (y_i)_{i = 1}^n$ be the PINN prediction and $y' = (y'_i)_{i = 1}^n$ the ground truth. Define
\begin{align*}
    \mathrm{L2RE} = \sqrt{\frac{\sum_{i = 1}^n (y_i - y'_i)^2}{\sum_{i = 1}^n y'^2_i}} = \sqrt{\frac{\|y - y'\|_2^2}{\|y'\|_2^2}}.
\end{align*}
We compute the L2RE using all points in the $255 \times 100$ grid on the interior of the problem domain, along with the 257 and 101 points used for the initial and boundary conditions.

We develop our experiments in PyTorch 2.0.0 \cite{paszke2019pytorch} with Python 3.10.12.
Each experiment is run on a single NVIDIA Titan V GPU using CUDA 11.8.
The code for our experiments is available at \href{https://github.com/pratikrathore8/opt_for_pinns}{https://github.com/pratikrathore8/opt\_for\_pinns}.


\section{Related Work}
Here we review common approaches for solving PDEs with physics-informed machine learning and PINN training strategies proposed in the literature.

\subsection{Physics-informed ML for Solving PDEs}
A variety of ML-based methods for solving PDEs have been proposed, including PINNs \cite{raissi2019unified}, the Fourier Neural Operator (FNO) \cite{li2021fourier}, and DeepONet \cite{lu2021deeponet}.
The PINN approach solves the PDE by 
using the loss function to penalize deviations from the PDE residual, boundary, and initial conditions.
Notably, PINNs do not require knowledge of the solution to solve the forward PDE problem.
On the other hand, the FNO and DeepONet sample and learn from known solutions to a parameterized class of PDEs to solve PDEs with another fixed value of the parameter. 
However, these operator learning approaches may not produce predictions consistent with the underlying physical laws that produced the data, which has led to the development of hybrid approaches such as physics-informed DeepONet \cite{wang2021learning}.
Our theory shows that the ill-conditioning issues we study in PINNs are unavoidable for any ML-based approach that penalizes deviations from the known physical laws.

\subsection{Challenges in Training PINNs}
The vanilla PINN \cite{raissi2019unified} can perform poorly when trying to solve high-dimensional, non-linear, and/or multi-scale PDEs.
Researchers have proposed a variety of modifications to the vanilla PINN to address these issues, many of which attempt to make the optimization problem easier to solve. 
\citet{wang2021understanding,wang2022respecting,wang2022when,nabian2021efficient,wu2023comprehensive,wu2023effective} propose loss reweighting/resampling to balance different components of the loss, \citet{yao2023multiadam,muller2023achieving} propose scale-invariant and natural gradient-based optimizers for PINN training, \citet{jagtap2020adaptive,jagtap2020locally,CiCP-34-869} propose adaptive activation functions which can accelerate convergence of the optimizer, and \citet{liu2024preconditioning} propose an approach to precondition the PINN loss itself. 
Other approaches include innovative loss functions and regularizations \cite{e2018deep,lu2021physics,kharazmi2021hpvpinns,khodayimehr2020varnet,yu2022gradient} and new architectures \cite{jagtap2020conservative,jagtap2020extended,li2020d3m,moseley2023fbpinn}.
These strategies work with varying degrees of success, and no single strategy improves performance across all PDEs.

Our work attempts to understand and tame the ill-conditioning in the (vanilla) PINN loss directly. 
We expect our ideas to work well with many of the above training strategies for PINNs; none of these training strategies rid the objective of the differential operator that generates the ill-conditioning in the PINN loss (with the possible exception of \citet{liu2024preconditioning}). 
However, \citet{liu2024preconditioning} preconditions the PINN loss directly, which is equivalent to left preconditioning, 
while our work studies the effects of preconditioned optimization methods on the PINN loss, which is equivalent to right preconditioning (\cref{subsec:how_lbfgs_preconditions}).
There is potential in combining the approach of \citet{liu2024preconditioning} and our approach to obtain a more reliable framework for training PINNs. 

Our work analyzes the spectrum (eigenvalues) of the Hessian $H_L$ of the loss.
Previous work \cite{wang2022when} studies the conditioning of the loss using the neural tangent kernel (NTK), which requires an infinite-width assumption on the neural network; our work studies the conditioning of the loss through the lens of the Hessian and yields useful results for finite-width PINN architectures.
Several works have also studied the \textit{spectral bias} of PINNs \cite{wang2021eigenvector,wang2022when,moseley2023fbpinn}, which refers to the inability of neural networks to learn high-frequency functions.
Note that our paper uses the word spectrum to refer to the Hessian eigenvalues, not the spectrum of the PDE solution.





\section{Good Solutions Require Near-zero Loss}
\label{sec:near_zero_loss}

\begin{figure*}[t]
    \centering
    \includegraphics[scale=0.33]{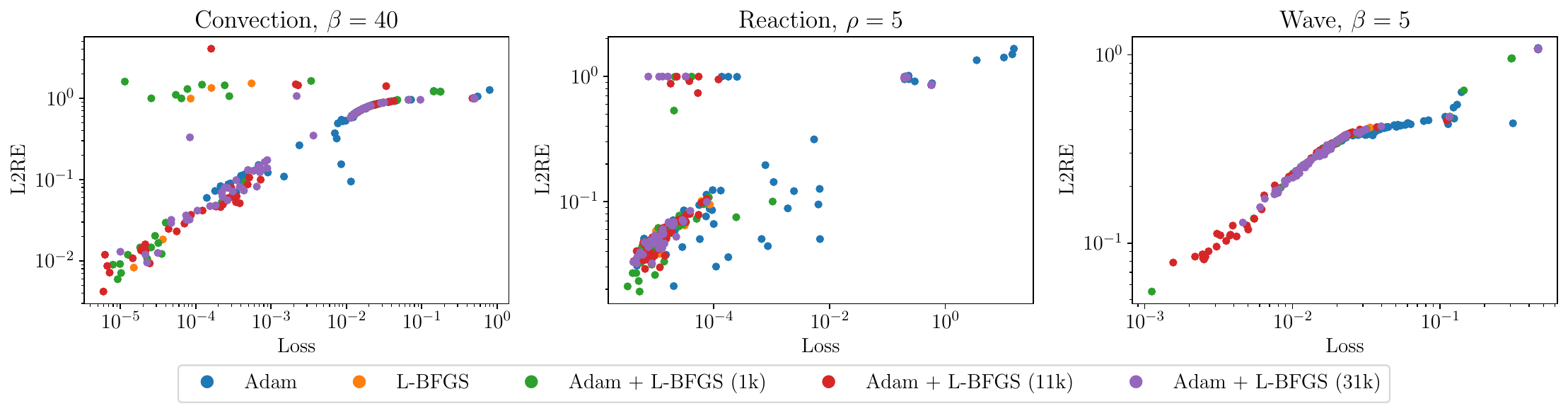}
    \caption{We plot the final L2RE against the final loss for each combination of network width, optimization strategy, and random seed.
    Across all three PDEs, a lower loss generally corresponds to a lower L2RE.}
    \label{fig:l2re_vs_loss}
\end{figure*}

First, we show that PINNs must be trained to near-zero loss to obtain a reasonably low L2RE.
This phenomenon can be observed in \cref{fig:l2re_vs_loss}, demonstrating that a lower loss generally corresponds to a lower L2RE.
For example, on the convection PDE, a loss of $10^{-3}$ yields an L2RE around $10^{-1}$, but decreasing the loss by a factor of $100$ to $10^{-5}$ yields an L2RE around $10^{-2}$, a 10$\times$ improvement.
This relationship between loss and L2RE in \cref{fig:l2re_vs_loss} 
is typical of many PDEs \cite{lu2022multifidelity}. 

The relationship in \cref{fig:l2re_vs_loss} underscores that high-accuracy optimization is required for a useful PINN.
There are instances (especially on the reaction ODE), where the PINN solution has a L2RE around 1, despite a near-zero loss; we provide insight into why this is occurring in \cref{sec:low_loss_high_l2re}. 
In \cref{sec:loss_landscape,sec:under_optimized}, we show that ill-conditioning and under-optimization make reaching a solution with sufficient accuracy difficult.


\section{The Loss Landscape is Ill-conditioned}
\label{sec:loss_landscape}

\begin{figure*}[t]
    \centering
    \includegraphics[scale=0.33]{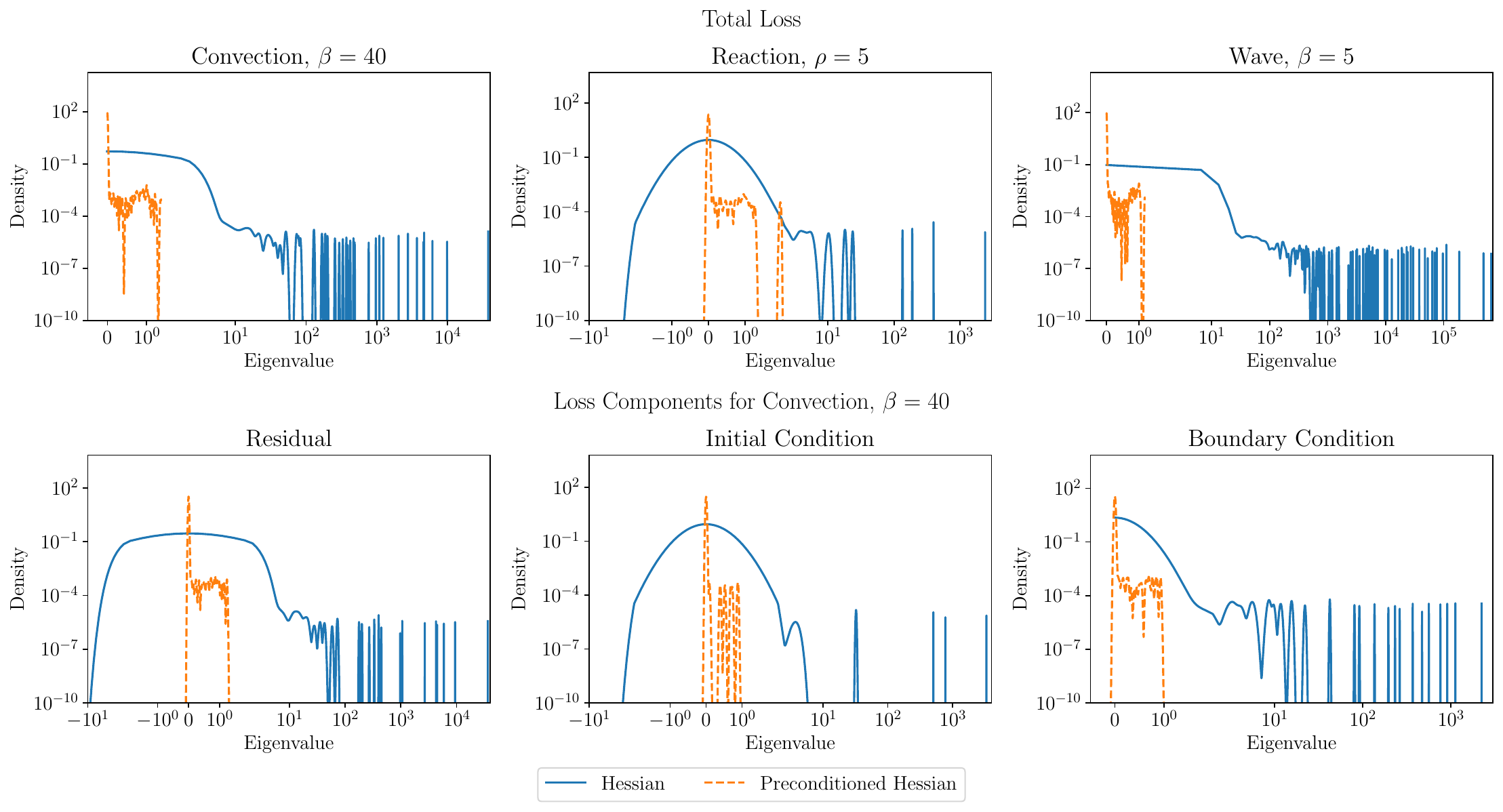}
    \caption{(Top) Spectral density of the Hessian and the preconditioned Hessian after 41000 iterations of \al{}. 
    The plots show that the PINN loss is ill-conditioned and that \lbfgs{} improves the conditioning, reducing the top eigenvalue by $10^3$ or more. \\
    (Bottom) Spectral density of the Hessian and the preconditioned Hessian of each loss component after 41000 iterations of \al{} for convection. The plots show that each component loss is ill-conditioned and that the conditioning is improved by \lbfgs{}.}
    \label{fig:spectral_density_multi_pde_convection_combined}
\end{figure*}

We show empirically that the ill-conditioning of the PINN loss is mainly due to the residual loss, which contains the differential operator.
We also show that quasi-Newton methods like \lbfgs{} improve the conditioning of the problem. 

\subsection{The PINN Loss is Ill-conditioned}


The conditioning of the loss $L$ plays a key role in the performance of first-order optimization methods \cite{nesterov2018lectures}.
We can understand the conditioning of an optimization problem through the eigenvalues of the Hessian of the loss, $H_L$. 
Intuitively, the eigenvalues of $H_L$ provide information about the local curvature of the loss function at a given point along different directions. 
The condition number is defined as the ratio of the largest magnitude's eigenvalue to the smallest magnitude's eigenvalue.
A large condition number implies the loss is very steep in some directions and flat in others, making it difficult for first-order methods to make sufficient progress toward the minimum. 
When $H_L(w)$ has a large condition number (particularly, for $w$ near the optimum), the loss $L$ is called \emph{ill-conditioned}.
For example, the convergence rate of gradient descent (GD) depends on the condition number \cite{nesterov2018lectures}, which results in GD converging slowly on ill-conditioned problems.


To investigate the conditioning of the PINN loss $L$, 
we would like to examine the eigenvalues of the Hessian. 
For large matrices, it is convenient to visualize the set of eigenvalues via  \emph{spectral density}, which approximates the distribution of the eigenvalues.
Fast approximation methods for the spectral density of the Hessian are available for deep neural networks \cite{ghorbani2019an, yao2020pyhessian}. 
\cref{fig:spectral_density_multi_pde_convection_combined} shows the estimated Hessian spectral density (solid lines) of the PINN loss for the convection, reaction, and wave problems after training with \al{}. 
For all three problems, we observe large outlier eigenvalues ($> 10^4$ for convection, $> 10^3$ for reaction, and $> 10^5$ for wave) in the spectrum, and a significant spectral density near $0$, implying that the loss $L$ is ill-conditioned.
The plots also show how the spectrum is improved by preconditioning (\cref{subsec:lbfgs_improvement}).

\subsection{The Ill-conditioning is Due to the Residual Loss}


We use the same method to study the conditioning of each component of the PINN loss. \cref{fig:spectral_density_multi_pde_convection_combined,fig:spectral_density_reaction_wave} show the estimated spectral density of the Hessian of the residual, initial condition, and boundary condition components of the PINN loss for each problem after training with \al. 
We see residual loss, which contains the differential operator $\mathcal D$, is the most ill-conditioned among all components.
Our theory (\cref{sec:theory}) shows this ill-conditioning is likely due to the ill-conditioning of $\mathcal D$.

\subsection{\lbfgs{} Improves Problem Conditioning}
\label{subsec:lbfgs_improvement}


Preconditioning is a popular technique for improving conditioning in optimization. 
A classic example is Newton's method, which uses second-order information (i.e., the Hessian) to (locally) transform an ill-conditioned loss landscape into a well-conditioned one.
\lbfgs{} is a quasi-Newton method that improves conditioning without explicit access to the problem Hessian. 
To examine the effectiveness of quasi-Newton methods for optimizing $L$, 
we compute the spectral density of the Hessian after \lbfgs{} preconditioning. (For details of this computation and how L-BFGS preconditions, see \cref{sec:lbfgs_spectral_info}.)
\cref{fig:spectral_density_multi_pde_convection_combined} shows this preconditioned Hessian spectral density (dashed lines). 
For all three problems, the magnitude of eigenvalues and the condition number has been reduced by at least $10^3$. 
In addition, the preconditioner improves the conditioning of each individual loss component of $L$ (\cref{fig:spectral_density_multi_pde_convection_combined,fig:spectral_density_reaction_wave}). 
These observations offer clear evidence that quasi-Newton methods improve the conditioning of the loss, and show the importance of quasi-Newton methods in training PINNs, which we demonstrate in \cref{sec:opt_comparison}. 
\section{\al{} Optimizes the Loss Better Than Other Methods}
\label{sec:opt_comparison}
We demonstrate that the combined optimization method \al{} consistently provides a smaller loss and L2RE than using Adam or \lbfgs{} alone.
We justify this finding using intuition from optimization theory.


\subsection{\al{} vs Adam or \lbfgs}
\cref{fig:opt_comparison} in \cref{sec:opt_comparison_additional} compares \al, Adam, and \lbfgs{} on the convection, reaction, and wave problems at difficult coefficient settings noted in the literature \cite{krishnapriyan2021characterizing, wang2022when}.
Across each network width, the lowest loss and L2RE is always delivered by \al.
Similarly, the lowest median loss and L2RE are almost always delivered by \al{} (\cref{fig:opt_comparison}).
The only exception is the reaction problem, where Adam outperforms \al{} on loss at width = 100 and L2RE at width = 200 (\cref{fig:opt_comparison}).

\cref{tab:loss_l2re_comparison} summarizes the best performance of each optimizer.
Again, \al{} is better than running either Adam or L-BFGS alone.
Notably, \al{} attains 14.2$\times$ smaller L2RE than Adam on the convection problem and 6.07$\times$ smaller L2RE than \lbfgs{} on the wave problem.

\begin{table}[t]
    \caption{Lowest loss for Adam, \lbfgs, and \al{} across all network widths after hyperparameter tuning. 
    \al{} attains both smaller loss and L2RE vs. Adam or \lbfgs. 
    }
    \vskip 0.15in
    \centering
    \tiny
    \begin{tabular}{|c|c|c|c|c|c|c|c|} 
    \hline 
    \multirow{2}{*}{Optimizer} & \multicolumn{2}{c|}{Convection} & \multicolumn{2}{c|}{Reaction} & \multicolumn{2}{c|}{Wave} \\ \cline{2-7}
                               & Loss & L2RE & Loss & L2RE & Loss & L2RE \\ \hline 
    Adam                        & 1.40e-4     & 5.96e-2     & 4.73e-6     & 2.12e-2     & 2.03e-2     & 3.49e-1     \\ \hline 
    L-BFGS                      & 1.51e-5     & 8.26e-3     & 8.93e-6     & 3.83e-2     & 1.84e-2     & 3.35e-1     \\ \hline 
    \al                         & \textbf{5.95e-6}     & \textbf{4.19e-3}     & \textbf{3.26e-6}     & \textbf{1.92e-2}     & \textbf{1.12e-3}     & \textbf{5.52e-2}      \\ \hline
    \end{tabular}
    \label{tab:loss_l2re_comparison}
\end{table}

\subsection{Intuition From Optimization Theory}
The success of \al{} over Adam and \lbfgs{} can be explained by existing results in optimization theory.
In neural networks, saddle points typically outnumber local minima \cite{dauphin2014identifying,lee2019firstorder}.
A saddle point can never be a global minimum. 
We want to reach a global minimum when training PINNs.

Newton's method (which \lbfgs{} attempts to approximate) is attracted to saddle points \cite{dauphin2014identifying},  and quasi-Newton methods such as \lbfgs{} also converge to saddle points since they ignore negative curvature \cite{dauphin2014identifying}.
On the other hand, first-order methods such as gradient descent and AdaGrad \cite{duchi2011adaptive} avoid saddle points \cite{lee2019firstorder,antonakopoulos2022adagrad}.
We expect that (full-gradient) Adam also avoids saddles for similar reasons, although
we are not aware of such a result.

Alas, first-order methods converge slowly when the problem is ill-conditioned. 
This result generalizes the well-known slow convergence of conjugate gradient (CG)
for ill-conditioned linear systems:
 $\mathcal O (\sqrt \kappa \log(\frac{1}{\epsilon}))$ iterations to converge to an $\epsilon$-approximate solution of a system with condition number $\kappa$.
In optimization, an analogous notion of a condition number in a set $\mathcal S$ near a global minimum is given by $\kappa_{f}(\mathcal S) \coloneqq \sup_{w \in \mathcal S} \| H_f(w) \| / \mu$, where $\mu$ is the \PL-constant (see \cref{sec:theory}).
Then gradient descent requires $\mathcal O (\kappa_{f}(\mathcal S) \log(\frac{1}{\epsilon}))$ iterations to converge to an $\epsilon$-suboptimal point.
For PINNs, the condition number near a solution is often $> 10^{4}$ (\cref{fig:spectral_density_multi_pde_convection_combined}), which leads to slow convergence of first-order methods. 
However, Newton's method and L-BFGS can significantly reduce the condition number (\cref{fig:spectral_density_multi_pde_convection_combined}), which yields faster convergence. 


\al{} combines the best of both first- and second-order/quasi-Newton methods. 
By running Adam first, we avoid saddle points that would attract L-BFGS.
By running L-BFGS after Adam, 
we can reduce the condition number of the problem, which leads to faster local convergence.
\cref{fig:under_optimization_intro} exemplifies this, showing faster convergence of \al{} over Adam on the wave equation.

This intuition also explains why Adam sometimes performs as well as \al{} on the reaction problem.
\cref{fig:spectral_density_multi_pde_convection_combined} shows the largest eigenvalue of the reaction problem is around $10^{3}$, 
while the largest eigenvalues of the convection and wave problems are around $10^{4}$ and $10^{5}$, suggesting the reaction problem is less ill-conditioned.


\section{The Loss is Often Under-optimized}
\label{sec:under_optimized}
In \cref{sec:opt_comparison}, we show that \al{} improves on running Adam or \lbfgs{} alone.
However, even \al{} does not reach a critical point of the loss: the loss is still under-optimized.
We show that the loss and L2RE can be further improved by running a damped version of Newton's method.

\subsection{Why is the Loss Under-optimized?}
\cref{fig:under_optimization} shows the run of \al{} with smallest L2RE for each PDE.
For each run, \lbfgs{} stops making progress before reaching the maximum number of iterations.
\lbfgs{} uses \textit{strong Wolfe line search}, as it is needed to maintain the stability of \lbfgs{} \cite{nocedal2006numerical}.
\lbfgs{} often terminates because it cannot find a positive step size satisfying these conditions---we have observed several instances where \lbfgs{} picks a step size of zero (\cref{fig:line_search_multi_pde} in \cref{sec:under_optimization_additional}), leading to early stopping.
Perversely, \lbfgs{} stops in these cases without reaching a critical point: 
the gradient norm is around $10^{-2}$ or $10^{-3}$ 
(see the bottom row of \cref{fig:under_optimization}).
The gradient still contains useful information for improving the loss.

\subsection{NysNewton-CG (NNCG)}
\label{subsec:NNCG} 
We can avoid premature termination by using a damped version of Newton's method with \textit{Armijo line search}.
The Armijo conditions use only a subset of the strong Wolfe conditions.
Under only Armijo conditions, \lbfgs{} is unstable; we require a different 
approximation to the Hessian ($p\times p$ for a neural net with $p$ parameters) that does not require storing ($\mathcal O(p^2)$) or inverting ($\mathcal O(p^3)$) the Hessian.
Instead, we run a Newton-CG algorithm that solves for the Newton step using preconditioned conjugate gradient (PCG).
This algorithm can be implemented efficiently with Hessian-vector products. These can be computed $\mathcal{O}\left((\nres+\nbc)p\right)$ time \cite{pearlmutter1994fast}.
\cref{sec:loss_landscape} shows that the Hessian is ill-conditioned with fast spectral decay, so CG without preconditioning will converge slowly.
Hence we use Nystr\"{o}mPCG, a PCG method that is designed to solve linear systems with fast spectral decay \cite{frangella2023randomized}.
The resulting algorithm is called NysNewton-CG (abbreviated NNCG); a full description of the algorithm appears in \cref{sec:under_optimization_additional}.

\subsection{Performance of NNCG}
\begin{figure*}
    \centering
    \includegraphics[scale=0.33]{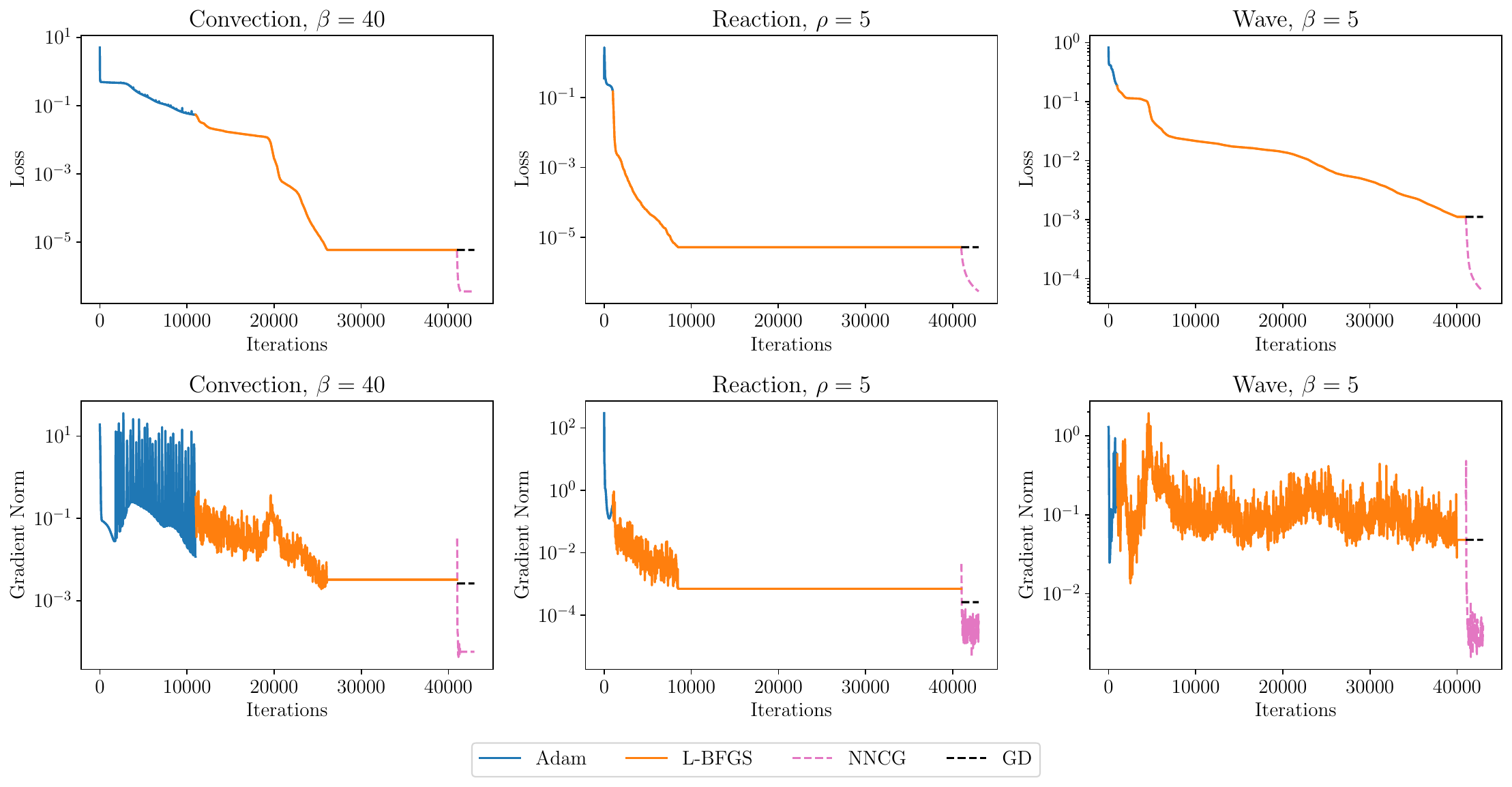}
    \caption{Performance of NNCG and GD after \al. 
    (Top) NNCG reduces the loss by a factor greater than 10 in all instances, while GD fails to make progress. (Bottom) Furthermore, NNCG significantly reduces the gradient norm on the convection and wave problems, while GD fails to do so.}
    \label{fig:under_optimization}
\end{figure*}

\begin{figure*}
    \centering
    \includegraphics[scale=0.33]{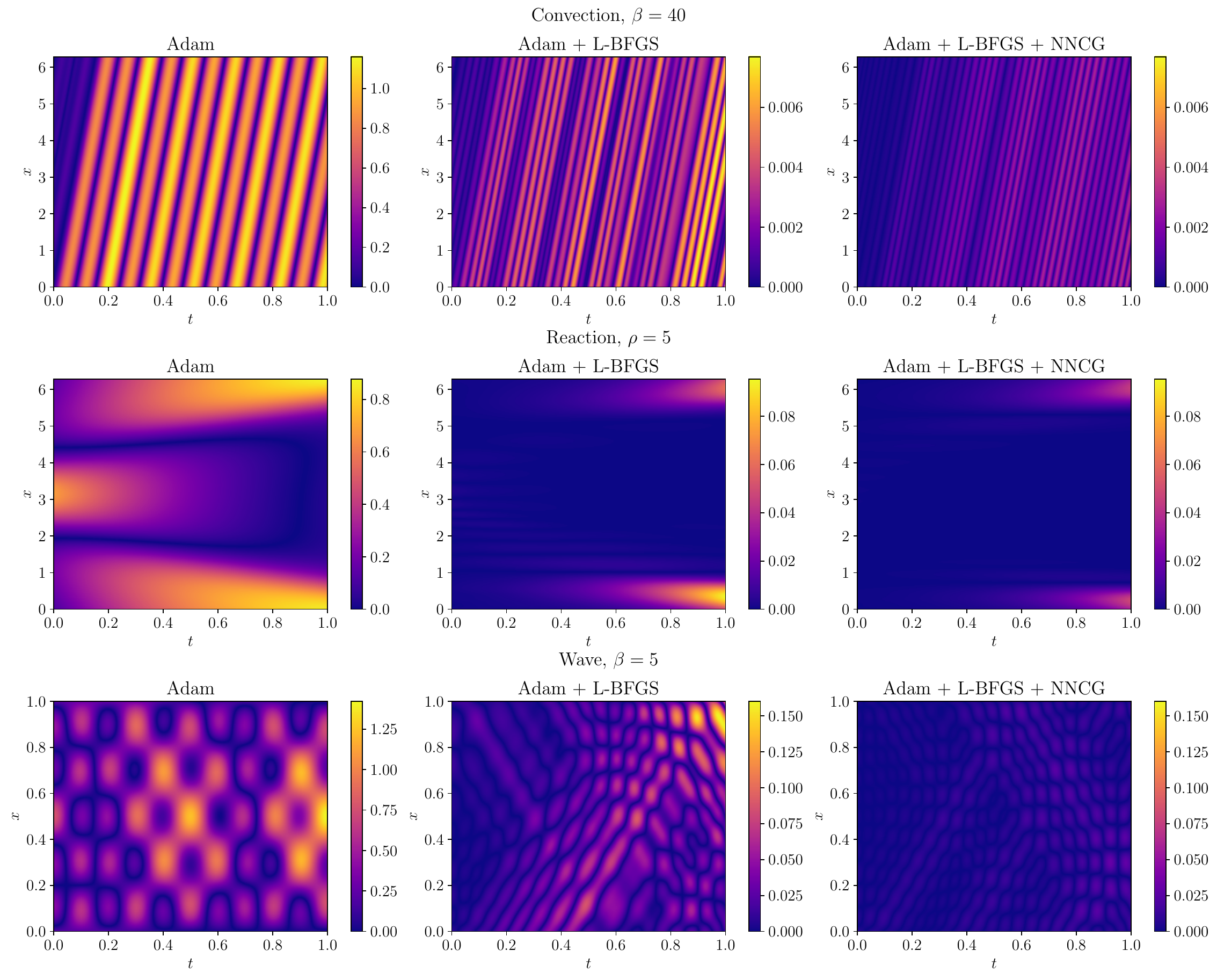}
    \caption{Absolute errors of the PINN solution at optimizer switch points. 
    The first column shows errors after Adam, the second column shows errors after running \lbfgs{} following Adam, and the third column shows the errors after running NNCG folllowing \al{}.
    \lbfgs{} improves the solution obtained from first running Adam, and NNCG further improves the solution even after \al{} stops making progress. 
    Note that Adam solution errors (left-most column) are presented at separate scales as these solutions are far off from the exact solutions. }
    \label{fig:solution_evolutions}
\end{figure*}


\begin{table*}[t]
    \caption{Loss and L2RE after fine-tuning by NNCG and GD. NNCG outperforms both GD and the original \al{} results.}
    \vskip 0.15in
    \centering
    \scriptsize
    \begin{tabular}{|c|c|c|c|c|c|c|c|} 
    \hline 
    \multirow{2}{*}{Optimizer} & \multicolumn{2}{c|}{Convection} & \multicolumn{2}{c|}{Reaction} & \multicolumn{2}{c|}{Wave} \\ \cline{2-7}
                               & Loss & L2RE & Loss & L2RE & Loss & L2RE \\ \hline 
        \al{} & 5.95e-6 & 4.19e-3 & 5.26e-6 & 1.92e-2 & 1.12e-3 & 5.52e-2 \\
        \hline
        \aln{} & \textbf{3.63e-7} & \textbf{1.94e-3} & \textbf{2.89e-7} & \textbf{9.92e-3} & \textbf{6.13e-5} & \textbf{1.27e-2} \\
        \hline
        \alg{} & 5.95e-6 & 4.19e-3 & 5.26e-6 & 1.92e-2 & 1.12e-3 & 5.52e-2 \\
        \hline
    \end{tabular}
    \label{tab:nncg_l2re_improvement}
\end{table*}

\cref{fig:under_optimization} shows that NNCG significantly improves both the loss and gradient norm of the solution when applied after \al{}, while \cref{fig:solution_evolutions} visualizes how NNCG improves the absolute error (pointwise) of the PINN solution when applied after \al{}.
Furthermore, \cref{tab:nncg_l2re_improvement} shows that NNCG also improves the L2RE of the PINN solution.
In contrast, applying gradient descent (GD) after \al{} improves neither the loss nor the L2RE. This result is unsurprising, as our theory predicts that NNCG will work better than GD for an ill-conditioned loss (\cref{sec:theory}). 

\subsection{Why Not Use NNCG Directly After Adam?}
\label{subsec:why_not_nncg}
Since NNCG improves the PINN solution and uses simpler line search conditions than \lbfgs, it is tempting to replace \lbfgs{} with NNCG entirely.
However, NNCG is slower than \lbfgs{}: the \lbfgs{} update can be computed in $\mathcal O(mp)$ time, where $m$ is the memory parameter, while just a single Hessian-vector product for computing the NNCG update requires $\mathcal{O}\left((\nres+\nbc)p\right)$ time. \cref{tab:wall_clock_time_comparison} shows NNCG takes 5, 20, and 322 more times per-iteration as \lbfgs{} on convection, reaction, and wave respectively. 
Consequently, we should run \al{} to make as much progress as possible before switching to NNCG.


\section{Theory}
\label{sec:theory} 

We relate the conditioning of the differential operator to the conditioning of the PINN loss function \eqref{eq:pinn_prob_gen}
in \cref{thm:informal_ill_cond}. 
When the differential operator is ill-conditioned, 
gradient descent takes many iterations to reach a high-precision solution. 
As a result, first-order methods alone may not deliver sufficient accuracy.

\begin{algorithm}[H]
	\centering
    \footnotesize
	\caption{Gradient-Damped Newton Descent (GDND)}
	\label{alg-GDND}
	\begin{algorithmic}
	\INPUT{$\#$ of gradient descent iterations $K_{\textup{GD}}$, gradient descent learning rate $\eta_{\textup{GD}}$, $\#$ of damped Newton iterations $K_{\textup{DN}}$, damped Newton learning rate $\eta_{\textup{DN}}$, damping parameter $\gamma$}
        \STATE{\textbf{Phase I}: \textbf{Gradient descent}} 
        \FOR{$k = 0,\dots,K_{\textup{GD}} - 1$}
            \STATE {$w_{k+1} = w_k-\eta_{\textup{GD}}\nabla L(w_k)$}
        \ENDFOR
        \STATE{\textbf{Phase II}: \textbf{Damped Newton}} 
        \STATE{Set $\tilde w_0 = w_{K_{\textup{GD}}}$}
		\FOR{$k = 0,\dots,K_{\textup{DN}} - 1$}
		  \STATE{$\tilde w_{k+1} = \tilde w_k-\eta_{\textup{DN}}\left(H_L(\tilde w_k)+\gamma I\right)^{-1}\nabla L(\tilde w_k)$}
		\ENDFOR
	\OUTPUT{approximate solution $\tilde w_{K_{\textup{DN}}}$}
	\end{algorithmic}
\end{algorithm}

To address this issue, 
we develop and analyze a hybrid algorithm, 
Gradient Damped Newton Descent (GDND, \cref{alg-GDND}),
that switches from gradient descent to damped Newton's method after a fixed number of iterations.
We show that GDND gives fast linear convergence independent of the condition number.
This theory supports our empirical results, 
which show that the best performance is obtained by running Adam and switching to L-BFGS.
Moreover, it provides a theoretical basis for using \aln{} to achieve the best performance. 

GDND differs from \aln{}, the algorithm we recommend in practice. 
We analyze GD instead of Adam because existing analyses of Adam \cite{defossez2022simple,zhang2022adam} do not mirror its empirical performance.
The reason we run both \lbfgs{} and damped Newton is to maximize computational efficiency (\cref{subsec:why_not_nncg}).

\subsection{Preliminaries}
We begin with the main assumption for our analysis.
\begin{assumption}[Interpolation]
\label{assp:interpolation}
    Let $\Wstar$ denote the set of minimizers of \eqref{eq:pinn_prob_gen}. 
    We assume that
    \[
    L(w_\star) = 0, \quad \text{for all}~w_\star \in \Wstar,
    \]
    i.e., the model perfectly fits the training data.
\end{assumption}
From a theoretical standpoint, 
\cref{assp:interpolation} is natural in light of various universal approximation theorems \cite{cybenko1989approximation,hornik1990universal,de2021approximation}, which show neural networks are capable of approximating any continuous function to arbitrary accuracy.  
Moreover, interpolation in neural networks is common in practice \cite{zhang2021understanding,belkin2021fit}.

\textbf{\PL-condition.}
In modern neural network optimization, the \PL-condition \cite{liu2022loss,liu2023aiming}
is key to showing convergence of gradient-based optimizers.
It is a local version of the celebrated Polyak-\L ojasiewicz condition \cite{polyak1963gradient,karimi2016linear}, specialized to interpolation. 
\begin{definition}[\PL-condition]
\label{def:pl_star}
    Suppose $L$ satisfies \cref{assp:interpolation}. 
    Let $\mathcal S\subset \R^p$. Then $L$ is $\mu$-\PL in $\mathcal S$ if 
    \[
    \frac{\|\nabla L(w)\|^2}{2\mu}\geq L(w), \qquad \forall w \in \mathcal S.
    \]
\end{definition}
The \PL-condition relates the gradient norm to the loss and implies that any minimizer in $\mathcal S$ is a global minimizer. 
Importantly, the \PL-condition can hold for non-convex losses and is known to hold, with high probability, 
for sufficiently wide neural nets with the least-squares loss \cite{liu2022loss}.

\begin{definition}[Condition number for \PL~loss functions]
\label{def:S_cond_num}
    Let $\mathcal S$ be a set for which $L$ is $\mu$-\PL. 
    Then the condition number of $L$ over $\mathcal S$ is given by
    \[
    \kappa_L(\mathcal S) = \frac{\sup_{w\in \mathcal S}\|H_L(w)\|}{\mu},
    \]
    where $H_L(w)$ is the Hessian matrix of the loss function.
\end{definition}
Gradient descent over $\mathcal S$ converges to $\epsilon$-suboptimality in $\mathcal O\left(\kappa_L(\mathcal S)\log\left(\frac{1}{\epsilon}\right)\right)$ iterations \cite{liu2022loss}.

\subsection{Ill-conditioned Differential Operators Lead to Challenging Optimization}
Here, we show that when the differential operator defining the PDE is $\emph{linear}$ and ill-conditioned, 
the condition number of the PINN objective (in the sense of \cref{def:S_cond_num}) is large. 
Our analysis in this regard is inspired by the recent work of \citet{de2023operator}, who prove a similar result for the population PINN residual loss.
However, \citet{de2023operator}'s analysis is based on the \emph{lazy training regime},
which assumes the NTK is approximately constant.
This regime does not accurately capture the behavior of practical neural networks \cite{allen2019can,chizat2019lazy,ghorbani2020neural,ghorbani2021linearized}.
Moreover, gradient descent can converge even with a non-constant NTK \cite{liu2020linearity}.
Our theoretical result is more closely aligned with deep learning practice as 
it does not assume lazy training
and pertains to the empirical loss rather than the population loss.

\cref{thm:informal_ill_cond} provides an informal version of our result in \cref{section:ill-cond-D} that shows that ill-conditioned differential operators induce ill-conditioning in the loss \eqref{eq:pinn_prob_gen}.
The theorem statement involves a kernel integral operator, $\Kc_{\infty}$ (defined in \eqref{eq:kern_op} in \cref{section:ill-cond-D}), evaluated at the optimum $w_\star$.
\begin{theorem}[Informal]
\label{thm:informal_ill_cond}
    Suppose \cref{assp:interpolation} holds and $p\geq \nres+\nbc$.  
    Fix $w_\star \in \Wstar$ and set $\A = \Dc^{*}\Dc$. 
    For some $\alpha>1/2$, suppose the eigenvalues of $\A\circ \Kc_{\infty}(w_\star)$ satisfy $\lambda_j(\mathcal A\circ \Kc_{\infty}(w_\star)) = \mathcal O\left(j^{-2\alpha}\right)$.
    If $\sqrt{n_{\textup{res}}} = \Omega\left(\log\left(\frac{1}{\delta}\right)\right)$, then for any set $\mathcal S$ that contains $w_\star$ and for which $L$ is $\mu$-\PL,
    \[
     \kappa_{L}(\mathcal S) = \Omega\left(n_{\textup{res}}^{\alpha}\right),\qquad \text{with probability} \geq 1-\delta.
    \]
\end{theorem}
\Cref{thm:informal_ill_cond} relates the conditioning of the PINN optimization problem to the conditioning of the operator $\A\circ\Kc_{\infty}(w_\star)$, where $\A$ is the Hermitian square of $\Dc$.
If the spectrum of $\A \circ \Kc_{\infty}(w_\star)$ decays polynomially,
then, with high probability, the condition number grows with $\nres$.
As $\nres$ typically ranges from $10^3$ to $10^4$, \cref{thm:informal_ill_cond} shows the condition number of the PINN problem is generally large, and so first-order methods will be slow to converge to the optimum. 
\cref{fig:condition_number_bound} in \cref{subsec:kappa_grows} empirically verifies the claim of \cref{thm:informal_ill_cond} for the convection equation.

\subsection{Efficient High-precision Solutions via GDND}

We now analyze the  convergence behavior of \cref{alg-GDND}. 
\cref{thm:GDND_informal} provides an informal version of our result in \cref{section:GDND_conv}.

\begin{theorem}[Informal]
\label{thm:GDND_informal}
    Suppose $L(w)$ satisfies the $\mu$-\PL-condition in a certain ball about $w_0$. 
    Then there exists $\eta_{\textup{GD}}>0$ and $K_{\textup{GD}}<\infty$ such that 
    Phase I of \cref{alg-GDND} outputs a point $w_{K_{\textup{GD}}}$, 
    for which Phase II of \cref{alg-GDND} with $\eta_{\textup{DN}} = 5/6$ and appropriate damping $\gamma>0$, satisfies
        \[
        L(\tilde w_{k})\leq \left(\frac{2}{3}\right)^{k}L(w_{K_{\textup{GD}}}).
        \]
        Hence after $K_{\textup{DN}} \geq 3\log\left(\frac{L(w_{K_{\textup{GD}}})}{\epsilon}\right)$ iterations, Phase II of \cref{alg-GDND} outputs a point satisfying
        $
        L(\tilde w_{K_{\textup{DN}}})\leq \epsilon.
        $
\end{theorem}
\cref{thm:GDND_informal} shows only a fixed number of gradient descent iterations are needed before \cref{alg-GDND} can switch to damped Newton's method and enjoy linear convergence independent of the condition number.
As the convergence rate of Phase II with damped Newton is independent of the condition number, 
\cref{alg-GDND} produces a highly accurate solution to \eqref{eq:pinn_prob_gen}.

Note that \cref{thm:GDND_informal} is local; \cref{alg-GDND} must find a point sufficiently close to a minimizer with gradient descent before switching to damped Newton's method and achieving rapid convergence.
It is not possible to develop a second-order method with a fast rate that does not require a good initialization,
as in the worst-case, global convergence of second-order methods may fail to improve over first-order methods \cite{cartis2010complexity,arjevani2019oracle}.
Moreover, \cref{thm:GDND_informal} is consistent with our experiments, which show \lbfgs{} is inferior to \al{}.

\section{Conclusion}
In this work, we explore the challenges posed by the loss landscape of PINNs for gradient-based optimizers.
We demonstrate ill-conditioning in the PINN loss and show it hinders effective training of PINNs.
By comparing Adam, L-BFGS, and \al{}, and introducing NNCG, we have demonstrated several approaches to improve the training process.
Our theory supports our experimental findings: we connect ill-conditioned differential operators to ill-conditioning in the PINN loss and prove the benefits of second-order methods over first-order methods for PINNs.

\section*{Acknowledgements}
We would like to acknowledge helpful comments from the anonymous reviewers and area chairs, which have improved this submission.
MU, PR, WL, and ZF gratefully acknowledge support from the National Science Foundation (NSF) Award IIS-2233762, the Office of Naval Research (ONR) Award N000142212825 and N000142312203, and the Alfred P. Sloan Foundation.
LL gratefully acknowledges support from the U.S. Department of Energy [DE-SC0022953].



\section*{Impact Statement}
This paper presents work whose goal is to advance the field of scientific machine learning. There are many potential societal consequences of our work, none which we feel must be specifically highlighted here.





\bibliography{references}

\begin{thebibliography}{74}
\providecommand{\natexlab}[1]{#1}
\providecommand{\url}[1]{\texttt{#1}}
\expandafter\ifx\csname urlstyle\endcsname\relax
  \providecommand{\doi}[1]{doi: #1}\else
  \providecommand{\doi}{doi: \begingroup \urlstyle{rm}\Url}\fi

\bibitem[Allen-Zhu \& Li(2019)Allen-Zhu and Li]{allen2019can}
Allen-Zhu, Z. and Li, Y.
\newblock What {C}an {R}es{N}et {L}earn {E}fficiently, {G}oing {B}eyond {K}ernels?
\newblock In \emph{Advances in Neural Information Processing Systems}, 2019.

\bibitem[Antonakopoulos et~al.(2022)Antonakopoulos, Mertikopoulos, Piliouras, and Wang]{antonakopoulos2022adagrad}
Antonakopoulos, K., Mertikopoulos, P., Piliouras, G., and Wang, X.
\newblock {A}da{G}rad {A}voids {S}addle {P}oints.
\newblock In \emph{Proceedings of the 39th International Conference on Machine Learning}, 2022.

\bibitem[Arjevani et~al.(2019)Arjevani, Shamir, and Shiff]{arjevani2019oracle}
Arjevani, Y., Shamir, O., and Shiff, R.
\newblock Oracle complexity of second-order methods for smooth convex optimization.
\newblock \emph{Mathematical Programming}, 178:\penalty0 327--360, 2019.

\bibitem[Bach(2013)]{bach2013sharp}
Bach, F.
\newblock Sharp analysis of low-rank kernel matrix approximations.
\newblock In \emph{Conference on learning theory}, 2013.

\bibitem[Belkin(2021)]{belkin2021fit}
Belkin, M.
\newblock Fit without fear: remarkable mathematical phenomena of deep learning through the prism of interpolation.
\newblock \emph{Acta Numerica}, 30:\penalty0 203--248, 2021.

\bibitem[Cartis et~al.(2010)Cartis, Gould, and Toint]{cartis2010complexity}
Cartis, C., Gould, I.~N., and Toint, P.~L.
\newblock On the complexity of steepest descent, {Newton}'s and regularized {Newton}'s methods for nonconvex unconstrained optimization problems.
\newblock \emph{SIAM Journal on Optimization}, 20\penalty0 (6):\penalty0 2833--2852, 2010.

\bibitem[Chizat et~al.(2019)Chizat, Oyallon, and Bach]{chizat2019lazy}
Chizat, L., Oyallon, E., and Bach, F.
\newblock On {L}azy {T}raining in {D}ifferentiable {P}rogramming.
\newblock In \emph{Advances in Neural Information Processing Systems}, 2019.

\bibitem[Cohen et~al.(2017)Cohen, Musco, and Musco]{cohen2017input}
Cohen, M.~B., Musco, C., and Musco, C.
\newblock Input sparsity time low-rank approximation via ridge leverage score sampling.
\newblock In \emph{Proceedings of the Twenty-Eighth Annual ACM-SIAM Symposium on Discrete Algorithms}, 2017.

\bibitem[Cuomo et~al.(2022)Cuomo, Di~Cola, Giampaolo, Rozza, Raissi, and Piccialli]{cuomo2022scientific}
Cuomo, S., Di~Cola, V.~S., Giampaolo, F., Rozza, G., Raissi, M., and Piccialli, F.
\newblock Scientific {M}achine {L}earning {T}hrough {P}hysics–{I}nformed {N}eural {N}etworks: {W}here {W}e {A}re and {W}hat’s {N}ext.
\newblock \emph{J. Sci. Comput.}, 92\penalty0 (3), 2022.

\bibitem[Cybenko(1989)]{cybenko1989approximation}
Cybenko, G.
\newblock Approximation by superpositions of a sigmoidal function.
\newblock \emph{Mathematics of control, signals and systems}, 2\penalty0 (4):\penalty0 303--314, 1989.

\bibitem[Dauphin et~al.(2014)Dauphin, Pascanu, Gulcehre, Cho, Ganguli, and Bengio]{dauphin2014identifying}
Dauphin, Y.~N., Pascanu, R., Gulcehre, C., Cho, K., Ganguli, S., and Bengio, Y.
\newblock Identifying and attacking the saddle point problem in high-dimensional non-convex optimization.
\newblock In \emph{Advances in Neural Information Processing Systems}, 2014.

\bibitem[De~Ryck et~al.(2021)De~Ryck, Lanthaler, and Mishra]{de2021approximation}
De~Ryck, T., Lanthaler, S., and Mishra, S.
\newblock On the approximation of functions by tanh neural networks.
\newblock \emph{Neural Networks}, 143:\penalty0 732--750, 2021.

\bibitem[De~Ryck et~al.(2023)De~Ryck, Bonnet, Mishra, and de~B{\'e}zenac]{de2023operator}
De~Ryck, T., Bonnet, F., Mishra, S., and de~B{\'e}zenac, E.
\newblock An operator preconditioning perspective on training in physics-informed machine learning.
\newblock \emph{arXiv preprint arXiv:2310.05801}, 2023.

\bibitem[D{\'e}fossez et~al.(2022)D{\'e}fossez, Bottou, Bach, and Usunier]{defossez2022simple}
D{\'e}fossez, A., Bottou, L., Bach, F., and Usunier, N.
\newblock A simple convergence proof of {Adam} and {Adagrad}.
\newblock \emph{Transactions on Machine Learning Research}, 2022.

\bibitem[Duchi et~al.(2011)Duchi, Hazan, and Singer]{duchi2011adaptive}
Duchi, J., Hazan, E., and Singer, Y.
\newblock Adaptive {S}ubgradient {M}ethods for {O}nline {L}earning and {S}tochastic {O}ptimization.
\newblock \emph{Journal of Machine Learning Research}, 12\penalty0 (61):\penalty0 2121--2159, 2011.

\bibitem[E \& Yu(2018)E and Yu]{e2018deep}
E, W. and Yu, B.
\newblock The {D}eep {R}itz {M}ethod: {A} {D}eep {L}earning-{B}ased {N}umerical {A}lgorithm for {S}olving {V}ariational {P}roblems.
\newblock \emph{Communications in Mathematics and Statistics}, 6\penalty0 (1):\penalty0 1--12, 2018.

\bibitem[Frangella et~al.(2023)Frangella, Tropp, and Udell]{frangella2023randomized}
Frangella, Z., Tropp, J.~A., and Udell, M.
\newblock Randomized {N}yström {P}reconditioning.
\newblock \emph{SIAM Journal on Matrix Analysis and Applications}, 44\penalty0 (2):\penalty0 718--752, 2023.

\bibitem[Ghorbani et~al.(2019)Ghorbani, Krishnan, and Xiao]{ghorbani2019an}
Ghorbani, B., Krishnan, S., and Xiao, Y.
\newblock An {I}nvestigation into {N}eural {N}et {O}ptimization via {H}essian {E}igenvalue {D}ensity.
\newblock In \emph{Proceedings of the 36th International Conference on Machine Learning}, 2019.

\bibitem[Ghorbani et~al.(2020)Ghorbani, Mei, Misiakiewicz, and Montanari]{ghorbani2020neural}
Ghorbani, B., Mei, S., Misiakiewicz, T., and Montanari, A.
\newblock When {D}o {N}eural {N}etworks {O}utperform {K}ernel {M}ethods?
\newblock In \emph{Advances in Neural Information Processing Systems}, 2020.

\bibitem[Ghorbani et~al.(2021)Ghorbani, Mei, Misiakiewicz, and Montanari]{ghorbani2021linearized}
Ghorbani, B., Mei, S., Misiakiewicz, T., and Montanari, A.
\newblock Linearized two-layers neural networks in high dimension.
\newblock \emph{The Annals of Statistics}, 49\penalty0 (2):\penalty0 1029--1054, 2021.

\bibitem[Glorot \& Bengio(2010)Glorot and Bengio]{glorot2010understanding}
Glorot, X. and Bengio, Y.
\newblock Understanding the difficulty of training deep feedforward neural networks.
\newblock In \emph{Proceedings of the Thirteenth International Conference on Artificial Intelligence and Statistics}, 2010.

\bibitem[Golub \& Meurant(2009)Golub and Meurant]{golub2009matrices}
Golub, G.~H. and Meurant, G.
\newblock \emph{Matrices, moments and quadrature with applications}, volume~30.
\newblock Princeton University Press, 2009.

\bibitem[Hao et~al.(2023)Hao, Yao, Su, Su, Wang, Lu, Xia, Zhang, Liu, Lu, and Zhu]{hao2023pinnacle}
Hao, Z., Yao, J., Su, C., Su, H., Wang, Z., Lu, F., Xia, Z., Zhang, Y., Liu, S., Lu, L., and Zhu, J.
\newblock {PINN}acle: {A} {C}omprehensive {B}enchmark of {P}hysics-{I}nformed {N}eural {N}etworks for {S}olving {PDE}s.
\newblock \emph{arXiv preprint arXiv:2306.08827}, 2023.

\bibitem[Horn \& Johnson(2012)Horn and Johnson]{horn2012matrix}
Horn, R.~A. and Johnson, C.~R.
\newblock \emph{Matrix Analysis}.
\newblock Cambridge University Press, 2nd edition, 2012.

\bibitem[Hornik et~al.(1990)Hornik, Stinchcombe, and White]{hornik1990universal}
Hornik, K., Stinchcombe, M., and White, H.
\newblock Universal approximation of an unknown mapping and its derivatives using multilayer feedforward networks.
\newblock \emph{Neural networks}, 3\penalty0 (5):\penalty0 551--560, 1990.

\bibitem[Jagtap \& Karniadakis(2020)Jagtap and Karniadakis]{jagtap2020extended}
Jagtap, A.~D. and Karniadakis, G.~E.
\newblock Extended physics-informed neural networks (xpinns): A generalized space-time domain decomposition based deep learning framework for nonlinear partial differential equations.
\newblock \emph{Communications in Computational Physics}, 28\penalty0 (5):\penalty0 2002--2041, 2020.

\bibitem[Jagtap et~al.(2020{\natexlab{a}})Jagtap, Kawaguchi, and Karniadakis]{jagtap2020adaptive}
Jagtap, A.~D., Kawaguchi, K., and Karniadakis, G.~E.
\newblock Adaptive activation functions accelerate convergence in deep and physics-informed neural networks.
\newblock \emph{Journal of Computational Physics}, 404:\penalty0 109136, 2020{\natexlab{a}}.

\bibitem[Jagtap et~al.(2020{\natexlab{b}})Jagtap, Kawaguchi, and Karniadakis]{jagtap2020locally}
Jagtap, A.~D., Kawaguchi, K., and Karniadakis, G.~E.
\newblock Locally adaptive activation functions with slope recovery for deep and physics-informed neural networks.
\newblock \emph{Proceedings of the Royal Society A: Mathematical, Physical and Engineering Sciences}, 2020{\natexlab{b}}.

\bibitem[Jagtap et~al.(2020{\natexlab{c}})Jagtap, Kharazmi, and Karniadakis]{jagtap2020conservative}
Jagtap, A.~D., Kharazmi, E., and Karniadakis, G.~E.
\newblock Conservative physics-informed neural networks on discrete domains for conservation laws: Applications to forward and inverse problems.
\newblock \emph{Computer Methods in Applied Mechanics and Engineering}, 365:\penalty0 113028, 2020{\natexlab{c}}.

\bibitem[Karimi et~al.(2016)Karimi, Nutini, and Schmidt]{karimi2016linear}
Karimi, H., Nutini, J., and Schmidt, M.
\newblock Linear {C}onvergence of {G}radient and {P}roximal-{G}radient {M}ethods under the {P}olyak-{{\L}ojasiewicz} {C}ondition.
\newblock In \emph{Machine Learning and Knowledge Discovery in Databases}, 2016.

\bibitem[Karniadakis et~al.(2021)Karniadakis, Kevrekidis, Lu, Perdikaris, Wang, and Yang]{karniadakis2021physicsinformed}
Karniadakis, G.~E., Kevrekidis, I.~G., Lu, L., Perdikaris, P., Wang, S., and Yang, L.
\newblock Physics-informed machine learning.
\newblock \emph{Nature Reviews Physics}, 3\penalty0 (6):\penalty0 422--440, 2021.

\bibitem[Kharazmi et~al.(2021)Kharazmi, Zhang, and Karniadakis]{kharazmi2021hpvpinns}
Kharazmi, E., Zhang, Z., and Karniadakis, G.~E.
\newblock hp-{VPINN}s: Variational physics-informed neural networks with domain decomposition.
\newblock \emph{Computer Methods in Applied Mechanics and Engineering}, 374:\penalty0 113547, 2021.

\bibitem[Khodayi-Mehr \& Zavlanos(2020)Khodayi-Mehr and Zavlanos]{khodayimehr2020varnet}
Khodayi-Mehr, R. and Zavlanos, M.
\newblock Var{N}et: {V}ariational {N}eural {N}etworks for the {S}olution of {P}artial {D}ifferential {E}quations.
\newblock In \emph{Proceedings of the 2nd Conference on Learning for Dynamics and Control}, pp.\  298--307, 2020.

\bibitem[Kingma \& Ba(2014)Kingma and Ba]{kingma2014adam}
Kingma, D.~P. and Ba, J.
\newblock Adam: A method for stochastic optimization.
\newblock \emph{arXiv preprint arXiv:1412.6980}, 2014.

\bibitem[Krishnapriyan et~al.(2021)Krishnapriyan, Gholami, Zhe, Kirby, and Mahoney]{krishnapriyan2021characterizing}
Krishnapriyan, A., Gholami, A., Zhe, S., Kirby, R., and Mahoney, M.~W.
\newblock Characterizing possible failure modes in physics-informed neural networks.
\newblock In \emph{Advances in Neural Information Processing Systems}, 2021.

\bibitem[Lee et~al.(2019)Lee, Panageas, Piliouras, Simchowitz, Jordan, and Recht]{lee2019firstorder}
Lee, J.~D., Panageas, I., Piliouras, G., Simchowitz, M., Jordan, M.~I., and Recht, B.
\newblock First-order methods almost always avoid strict saddle points.
\newblock \emph{Mathematical Programming}, 176\penalty0 (1):\penalty0 311--337, 2019.

\bibitem[Li et~al.(2020)Li, Tang, Wu, and Liao]{li2020d3m}
Li, K., Tang, K., Wu, T., and Liao, Q.
\newblock {D3M}: A {D}eep {D}omain {D}ecomposition {M}ethod for {P}artial {D}ifferential {E}quations.
\newblock \emph{IEEE Access}, 8:\penalty0 5283--5294, 2020.

\bibitem[Li et~al.(2021)Li, Kovachki, Azizzadenesheli, liu, Bhattacharya, Stuart, and Anandkumar]{li2021fourier}
Li, Z., Kovachki, N.~B., Azizzadenesheli, K., liu, B., Bhattacharya, K., Stuart, A., and Anandkumar, A.
\newblock Fourier {N}eural {O}perator for {P}arametric {P}artial {D}ifferential {E}quations.
\newblock In \emph{International Conference on Learning Representations}, 2021.

\bibitem[Lin et~al.(2016)Lin, Saad, and Yang]{lin2016approximating}
Lin, L., Saad, Y., and Yang, C.
\newblock Approximating spectral densities of large matrices.
\newblock \emph{SIAM review}, 58\penalty0 (1):\penalty0 34--65, 2016.

\bibitem[Liu et~al.(2020)Liu, Zhu, and Belkin]{liu2020linearity}
Liu, C., Zhu, L., and Belkin, M.
\newblock On the linearity of large non-linear models: when and why the tangent kernel is constant.
\newblock \emph{Advances in Neural Information Processing Systems}, 2020.

\bibitem[Liu et~al.(2022)Liu, Zhu, and Belkin]{liu2022loss}
Liu, C., Zhu, L., and Belkin, M.
\newblock Loss landscapes and optimization in over-parameterized non-linear systems and neural networks.
\newblock \emph{Applied and Computational Harmonic Analysis}, 59:\penalty0 85--116, 2022.

\bibitem[Liu et~al.(2023)Liu, Drusvyatskiy, Belkin, Davis, and Ma]{liu2023aiming}
Liu, C., Drusvyatskiy, D., Belkin, M., Davis, D., and Ma, Y.-A.
\newblock Aiming towards the minimizers: fast convergence of {SGD} for overparametrized problems.
\newblock \emph{arXiv preprint arXiv:2306.02601}, 2023.

\bibitem[Liu \& Nocedal(1989)Liu and Nocedal]{liu1989limited}
Liu, D.~C. and Nocedal, J.
\newblock On the limited memory {BFGS} method for large scale optimization.
\newblock \emph{Mathematical Programming}, 45\penalty0 (1):\penalty0 503--528, 1989.

\bibitem[Liu et~al.(2024)Liu, Su, Yao, Hao, Su, Wu, and Zhu]{liu2024preconditioning}
Liu, S., Su, C., Yao, J., Hao, Z., Su, H., Wu, Y., and Zhu, J.
\newblock Preconditioning for physics-informed neural networks, 2024.

\bibitem[Lu et~al.(2021{\natexlab{a}})Lu, Jin, Pang, Zhang, and Karniadakis]{lu2021deeponet}
Lu, L., Jin, P., Pang, G., Zhang, Z., and Karniadakis, G.~E.
\newblock Learning nonlinear operators via {DeepONet} based on the universal approximation theorem of operators.
\newblock \emph{Nature Machine Intelligence}, 3\penalty0 (3):\penalty0 218--229, 2021{\natexlab{a}}.

\bibitem[Lu et~al.(2021{\natexlab{b}})Lu, Meng, Mao, and Karniadakis]{lu2021deepxde}
Lu, L., Meng, X., Mao, Z., and Karniadakis, G.~E.
\newblock {DeepXDE}: {A} {D}eep {L}earning {L}ibrary for {S}olving {D}ifferential {E}quations.
\newblock \emph{SIAM Review}, 63\penalty0 (1):\penalty0 208--228, 2021{\natexlab{b}}.

\bibitem[Lu et~al.(2021{\natexlab{c}})Lu, Pestourie, Yao, Wang, Verdugo, and Johnson]{lu2021physics}
Lu, L., Pestourie, R., Yao, W., Wang, Z., Verdugo, F., and Johnson, S.~G.
\newblock Physics-informed neural networks with hard constraints for inverse design.
\newblock \emph{SIAM Journal on Scientific Computing}, 43\penalty0 (6):\penalty0 B1105--B1132, 2021{\natexlab{c}}.

\bibitem[Lu et~al.(2022)Lu, Pestourie, Johnson, and Romano]{lu2022multifidelity}
Lu, L., Pestourie, R., Johnson, S.~G., and Romano, G.
\newblock Multifidelity deep neural operators for efficient learning of partial differential equations with application to fast inverse design of nanoscale heat transport.
\newblock \emph{Physical Review Research}, 4\penalty0 (2):\penalty0 023210, 2022.

\bibitem[Mishra \& Molinaro(2023)Mishra and Molinaro]{mishra2023estimates}
Mishra, S. and Molinaro, R.
\newblock Estimates on the generalization error of physics-informed neural networks for approximating pdes.
\newblock \emph{IMA Journal of Numerical Analysis}, 43\penalty0 (1):\penalty0 1--43, 2023.

\bibitem[Moseley et~al.(2023)Moseley, Markham, and Nissen-Meyer]{moseley2023fbpinn}
Moseley, B., Markham, A., and Nissen-Meyer, T.
\newblock Finite basis physics-informed neural networks ({FBPINNs}): a scalable domain decomposition approach for solving differential equations.
\newblock \emph{Advances in Computational Mathematics}, 49\penalty0 (4):\penalty0 62, 2023.

\bibitem[M\"{u}ller \& Zeinhofer(2023)M\"{u}ller and Zeinhofer]{muller2023achieving}
M\"{u}ller, J. and Zeinhofer, M.
\newblock Achieving {H}igh {A}ccuracy with {PINN}s via {E}nergy {N}atural {G}radient {D}escent.
\newblock In \emph{Proceedings of the 40th International Conference on Machine Learning}, 2023.

\bibitem[Nabian et~al.(2021)Nabian, Gladstone, and Meidani]{nabian2021efficient}
Nabian, M.~A., Gladstone, R.~J., and Meidani, H.
\newblock Efficient training of physics‐informed neural networks via importance sampling.
\newblock \emph{Comput.-Aided Civ. Infrastruct. Eng.}, 36\penalty0 (8):\penalty0 962–977, 2021.

\bibitem[Nesterov(2018)]{nesterov2018lectures}
Nesterov, Y.
\newblock \emph{Lectures on Convex Optimization}.
\newblock Springer Publishing Company, Incorporated, 2nd edition, 2018.

\bibitem[Nocedal \& Wright(2006)Nocedal and Wright]{nocedal2006numerical}
Nocedal, J. and Wright, S.~J.
\newblock \emph{Numerical Optimization}.
\newblock Springer, 2nd edition, 2006.

\bibitem[Paszke et~al.(2019)Paszke, Gross, Massa, Lerer, Bradbury, Chanan, Killeen, Lin, Gimelshein, Antiga, Desmaison, K{\"{o}}pf, Yang, DeVito, Raison, Tejani, Chilamkurthy, Steiner, Fang, Bai, and Chintala]{paszke2019pytorch}
Paszke, A., Gross, S., Massa, F., Lerer, A., Bradbury, J., Chanan, G., Killeen, T., Lin, Z., Gimelshein, N., Antiga, L., Desmaison, A., K{\"{o}}pf, A., Yang, E.~Z., DeVito, Z., Raison, M., Tejani, A., Chilamkurthy, S., Steiner, B., Fang, L., Bai, J., and Chintala, S.
\newblock Py{T}orch: {A}n {I}mperative {S}tyle, {H}igh-{P}erformance {D}eep {L}earning {L}ibrary.
\newblock \emph{arXiv preprint arXiv:1912.01703}, 2019.

\bibitem[Pearlmutter(1994)]{pearlmutter1994fast}
Pearlmutter, B.~A.
\newblock Fast exact multiplication by the hessian.
\newblock \emph{Neural computation}, 6\penalty0 (1):\penalty0 147--160, 1994.

\bibitem[Polyak(1963)]{polyak1963gradient}
Polyak, B.~T.
\newblock Gradient methods for minimizing functionals.
\newblock \emph{Zhurnal vychislitel’noi matematiki i matematicheskoi fiziki}, 3\penalty0 (4):\penalty0 643--653, 1963.

\bibitem[Raissi et~al.(2019)Raissi, Perdikaris, and Karniadakis]{raissi2019unified}
Raissi, M., Perdikaris, P., and Karniadakis, G.
\newblock Physics-informed neural networks: A deep learning framework for solving forward and inverse problems involving nonlinear partial differential equations.
\newblock \emph{Journal of Computational Physics}, 378:\penalty0 686--707, 2019.

\bibitem[Rohrhofer et~al.(2023)Rohrhofer, Posch, G{\"o}{\ss}nitzer, and Geiger]{rohrhofer2023on}
Rohrhofer, F.~M., Posch, S., G{\"o}{\ss}nitzer, C., and Geiger, B.~C.
\newblock On the {R}ole of {F}ixed {P}oints of {D}ynamical {S}ystems in {T}raining {P}hysics-{I}nformed {N}eural {N}etworks.
\newblock \emph{Transactions on Machine Learning Research}, 2023.

\bibitem[Rudi et~al.(2017)Rudi, Carratino, and Rosasco]{rudi2017falkon}
Rudi, A., Carratino, L., and Rosasco, L.
\newblock {FALKON}: An {O}ptimal {L}arge {S}cale {K}ernel {M}ethod.
\newblock In \emph{Advances in Neural Information Processing Systems}, 2017.

\bibitem[Tropp(2015)]{tropp2015introduction}
Tropp, J.~A.
\newblock An introduction to matrix concentration inequalities.
\newblock \emph{Foundations and Trends{\textregistered} in Machine Learning}, 8\penalty0 (1-2):\penalty0 1--230, 2015.

\bibitem[Wang et~al.(2023)Wang, Lu, Song, and Huang]{CiCP-34-869}
Wang, H., Lu, L., Song, S., and Huang, G.
\newblock Learning {S}pecialized {A}ctivation {F}unctions for {P}hysics-{I}nformed {N}eural {N}etworks.
\newblock \emph{Communications in Computational Physics}, 34\penalty0 (4):\penalty0 869--906, 2023.

\bibitem[Wang et~al.(2021{\natexlab{a}})Wang, Teng, and Perdikaris]{wang2021understanding}
Wang, S., Teng, Y., and Perdikaris, P.
\newblock Understanding and {M}itigating {G}radient {F}low {P}athologies in {P}hysics-{I}nformed {N}eural {N}etworks.
\newblock \emph{SIAM Journal on Scientific Computing}, 43\penalty0 (5):\penalty0 A3055--A3081, 2021{\natexlab{a}}.

\bibitem[Wang et~al.(2021{\natexlab{b}})Wang, Wang, and Perdikaris]{wang2021eigenvector}
Wang, S., Wang, H., and Perdikaris, P.
\newblock On the eigenvector bias of {F}ourier feature networks: From regression to solving multi-scale {PDE}s with physics-informed neural networks.
\newblock \emph{Computer Methods in Applied Mechanics and Engineering}, 384:\penalty0 113938, 2021{\natexlab{b}}.

\bibitem[Wang et~al.(2021{\natexlab{c}})Wang, Wang, and Perdikaris]{wang2021learning}
Wang, S., Wang, H., and Perdikaris, P.
\newblock Learning the solution operator of parametric partial differential equations with physics-informed {DeepONets}.
\newblock \emph{Science Advances}, 7\penalty0 (40):\penalty0 eabi8605, 2021{\natexlab{c}}.

\bibitem[Wang et~al.(2022{\natexlab{a}})Wang, Sankaran, and Perdikaris]{wang2022respecting}
Wang, S., Sankaran, S., and Perdikaris, P.
\newblock Respecting causality is all you need for training physics-informed neural networks.
\newblock \emph{arXiv preprint arXiv:2203.07404}, 2022{\natexlab{a}}.

\bibitem[Wang et~al.(2022{\natexlab{b}})Wang, Yu, and Perdikaris]{wang2022when}
Wang, S., Yu, X., and Perdikaris, P.
\newblock When and why {PINNs} fail to train: A neural tangent kernel perspective.
\newblock \emph{Journal of Computational Physics}, 449:\penalty0 110768, 2022{\natexlab{b}}.

\bibitem[Wu et~al.(2023{\natexlab{a}})Wu, Zhu, Tan, Kartha, and Lu]{wu2023comprehensive}
Wu, C., Zhu, M., Tan, Q., Kartha, Y., and Lu, L.
\newblock A comprehensive study of non-adaptive and residual-based adaptive sampling for physics-informed neural networks.
\newblock \emph{Computer Methods in Applied Mechanics and Engineering}, 403:\penalty0 115671, 2023{\natexlab{a}}.

\bibitem[Wu et~al.(2023{\natexlab{b}})Wu, Daneker, Jolley, Turner, and Lu]{wu2023effective}
Wu, W., Daneker, M., Jolley, M.~A., Turner, K.~T., and Lu, L.
\newblock Effective data sampling strategies and boundary condition constraints of physics-informed neural networks for identifying material properties in solid mechanics.
\newblock \emph{Applied mathematics and mechanics}, 44\penalty0 (7):\penalty0 1039--1068, 2023{\natexlab{b}}.

\bibitem[Yao et~al.(2023)Yao, Su, Hao, Liu, Su, and Zhu]{yao2023multiadam}
Yao, J., Su, C., Hao, Z., Liu, S., Su, H., and Zhu, J.
\newblock {M}ulti{A}dam: Parameter-wise {S}cale-invariant {O}ptimizer for {M}ultiscale {T}raining of {P}hysics-informed {N}eural {N}etworks.
\newblock In \emph{Proceedings of the 40th International Conference on Machine Learning}, 2023.

\bibitem[Yao et~al.(2020)Yao, Gholami, Keutzer, and Mahoney]{yao2020pyhessian}
Yao, Z., Gholami, A., Keutzer, K., and Mahoney, M.~W.
\newblock Py{H}essian: {N}eural {N}etworks {T}hrough the {L}ens of the {H}essian.
\newblock In \emph{2020 IEEE International Conference on Big Data (Big Data)}, 2020.

\bibitem[Yu et~al.(2022)Yu, Lu, Meng, and Karniadakis]{yu2022gradient}
Yu, J., Lu, L., Meng, X., and Karniadakis, G.~E.
\newblock Gradient-enhanced physics-informed neural networks for forward and inverse {PDE} problems.
\newblock \emph{Computer Methods in Applied Mechanics and Engineering}, 393:\penalty0 114823, 2022.

\bibitem[Zhang et~al.(2021)Zhang, Bengio, Hardt, Recht, and Vinyals]{zhang2021understanding}
Zhang, C., Bengio, S., Hardt, M., Recht, B., and Vinyals, O.
\newblock Understanding deep learning (still) requires rethinking generalization.
\newblock \emph{Communications of the ACM}, 64\penalty0 (3):\penalty0 107--115, 2021.

\bibitem[Zhang et~al.(2022)Zhang, Chen, Shi, Sun, and Luo]{zhang2022adam}
Zhang, Y., Chen, C., Shi, N., Sun, R., and Luo, Z.-Q.
\newblock Adam {C}an {C}onverge {W}ithout {A}ny {M}odification {O}n {U}pdate {R}ules.
\newblock In \emph{Advances in Neural Information Processing Systems}, 2022.

\end{thebibliography}
\bibliographystyle{icml2024}

\newpage
\appendix
\onecolumn
\section{Additional Details on Problem Setup}
\label{sec:problem_setup_additional}
Here we present the differential equations that we study in our experiments.

\subsection{Convection}
The one-dimensional convection problem is a hyperbolic PDE that can be used to model fluid flow, heat transfer, and biological processes.
The convection PDE we study is
\begin{align*}
    \frac{\partial u}{\partial t} + \beta \frac{\partial u}{\partial x} = 0, & \quad x \in (0, 2\pi), t \in (0, 1), \\
    u(x, 0) = \sin(x), & \quad x \in [0, 2\pi], \\
    u(0, t) = u(2 \pi, t), & \quad t \in [0, 1]. 
\end{align*}

The analytical solution to this PDE is $u(x, t) = \sin(x - \beta t)$.
We set $\beta = 40$ in our experiments.

\subsection{Reaction}
The one-dimensional reaction problem is a non-linear ODE which can be used to model chemical reactions.
The reaction ODE we study is
\begin{align*}
    \frac{\partial u}{\partial t} - \rho u (1 - u) = 0, & \quad x \in (0, 2\pi), t \in (0, 1) \\
    u(x, 0) = \exp \left( -\frac{(x - \pi)^2}{2 (\pi / 4)^2} \right), & \quad x \in [0, 2\pi], \\
    u(0, t) = u(2 \pi, t), & \quad t \in [0, 1].
\end{align*}

The analytical solution to this ODE is $u(x, t) = \frac{h(x) e^{\rho t}}{h(x) e^{\rho t} + 1 - h(x)}$, where $h(x) = \exp \left( -\frac{(x - \pi)^2}{2 (\pi / 4)^2} \right)$.
We set $\rho = 5$ in our experiments.



\subsection{Wave}
The one-dimensional wave problem is a hyperbolic PDE that often arises in acoustics, electromagnetism, and fluid dynamics.
The wave PDE we study is
\begin{align*}
    \frac{\partial^2 u}{\partial t^2} - 4 \frac{\partial^2 u}{\partial x^2} = 0, & \quad x \in (0, 1), t \in (0, 1), \\
    u(x, 0) = \sin(\pi x) + \frac{1}{2} \sin(\beta \pi x), & \quad x \in [0, 1], \\
    \frac{\partial u(x, 0)}{\partial t} = 0, & \quad x \in [0, 1], \\
    u(0, t) = u(1, t) = 0, & \quad t \in [0, 1].
\end{align*}

The analytical solution to this PDE is $u(x, t) = \sin(\pi x) \cos(2 \pi t) + \frac{1}{2} \sin(\beta \pi x) \cos(2 \beta \pi t)$.
We set $\beta = 5$ in our experiments.
\section{Why can Low Losses Correspond to Large L2RE?}
\label{sec:low_loss_high_l2re}

\begin{figure*}
    \centering
    \includegraphics[scale=0.33]{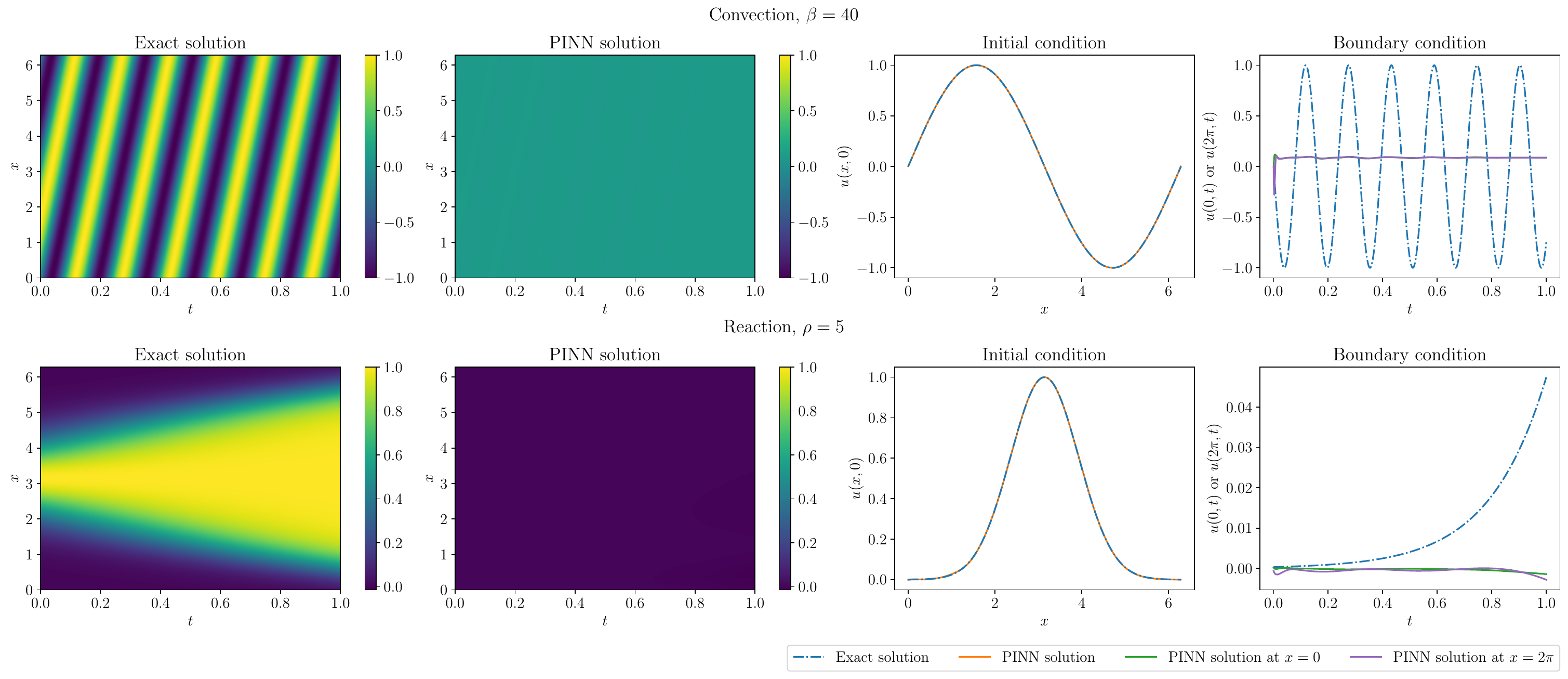}
    \caption{The first two columns from the left display the exact solutions and PINN solutions. 
    The PINN fails to learn the exact solution, which leads to large L2RE.
    Moreover, the PINN solutions are effectively constant over the domain.
    The third and fourth columns from the left display the PINN solutions at the initial time ($t = 0$) and the boundaries ($x = 0$ and $x = 2\pi$). 
    The PINN solutions learn the initial conditions, but they do not learn the boundary conditions.
    }
    \label{fig:high_l2re_solutions}
\end{figure*}

In \cref{fig:l2re_vs_loss}, there are several instances on the convection PDE and reaction ODE where the PINN loss is close to 0, but the L2RE of the PINN solution is close to 1.
\citet{rohrhofer2023on} demonstrate that PINNs can be attracted to points in the loss landscape that minimize the residual portion of the PINN loss, $\frac{1}{2\nres}\sum_{i=1}^{\nres}\left(\Dc[u(x_r^i; w),x_r^i]\right)^2$, to 0.
However, these can correspond to trivial solutions: for the convection PDE, the residual portion is equal to 0 for any constant function $u$; for the reaction ODE, the residual portion is equal to $0$ for constant $u = 0$ or $u = 1$.

To show that the PINN is indeed learning a trivial solution, we visualize two solutions with small residual loss but large L2RE in \cref{fig:high_l2re_solutions}.
The second column of \cref{fig:high_l2re_solutions} shows the PINN solutions are close to 0 almost everywhere in the domain. 
Interestingly, the PINN solutions correctly learn the initial condition.
However, the PINN solutions for the convection PDE and reaction ODE do not match the exact solution at the boundaries.
One approach for alleviating this training issue would be to (adaptively) reweight the residual, initial condition, and boundary condition terms in the PINN loss \cite{wang2021understanding,wang2022when}.

\section{Computing the Spectral Density of the \lbfgs{}-preconditioned Hessian}
\label{sec:lbfgs_spectral_info}

\subsection{How L-BFGS Preconditions} 
\label{subsec:how_lbfgs_preconditions}
To minimize \eqref{eq:pinn_prob_gen}, L-BFGS uses the update
\begin{equation}
\label{eq:l-bfgs}
w_{k+1} = w_k-\eta H_k \nabla L(w_k), 
\end{equation}
where $H_k$ is a matrix approximating the inverse Hessian.
We now show how \eqref{eq:l-bfgs} is equivalent to preconditioning the objective \eqref{eq:pinn_prob_gen}.
Define the coordinate transformation $w = H_k^{1/2}z$. 
By the chain rule, $\nabla L(z) = H_k^{1/2}\nabla L(w)$ and $H_L(z) = H^{1/2}_k H_L(w)H_k^{1/2}$. 
Thus, \eqref{eq:l-bfgs} is equivalent to
\begin{align}
\label{eq:precond}
    & z_{k+1} = z_k-\eta \nabla L(z_k), \\
    & w_{k+1} = H^{1/2}_kz_{k+1}. \nonumber
\end{align}

\cref{eq:precond} reveals how L-BFGS preconditions \eqref{eq:pinn_prob_gen}.
L-BFGS first takes a step in the \emph{preconditioned} $z$-space, where the conditioning is determined by $H_L(z)$, the preconditioned Hessian. 
Since $H_k$ approximates $H_L^{-1}(w)$, $H^{1/2}_k H_L(w) H_k^{1/2} \approx I_p$, so the condition number of $H_L(z)$ is much smaller than that of $H_L(w)$. 
Consequently, L-BFGS can take a step that makes more progress than a method like gradient descent, which performs no preconditioning at all. 
In the second phase, L-BFGS maps the progress in the preconditioned space back to the original space.
Thus, L-BFGS is able to make superior progress by transforming \eqref{eq:pinn_prob_gen} to another space where the conditioning is more favorable, which enables it to compute an update that better reduces the loss in \eqref{eq:pinn_prob_gen}.

\subsection{Preconditioned Spectral Density Computation}
Here we discuss how to compute the spectral density of the Hessian after preconditioning by L-BFGS.
This is the procedure we use to generate the figures in \cref{subsec:lbfgs_improvement}. 

\lbfgs{} stores a set of vector pairs given by the difference in consecutive iterates and gradients from most recent $m$ iterations (we use $m = 100$ in our experiments).
To compute the update direction $H_k \nabla f_k$, \lbfgs{} combines the stored vector pairs with a recursive scheme \cite{nocedal2006numerical}.
Defining
\[
    s_{k} = x_{k+1} - x_{k}, 
    \quad y_k = \nabla f_{k+1} - \nabla f_{k}, 
    \quad \rho_{k} = \frac{1}{y_{k}^{T}s_{k}}, 
    \quad \gamma_{k} = \frac{s_{k-1}^{T}y_{k-1}}{y_{k-1}^{T}y_{k-1}}, 
    \quad V_k = I - \rho_{k} y_{k} s_{k}^{T}, 
    \quad H_k^{0} = \gamma_{k} I,
\]
the formula for $H_k$ can be written as
\[
    H_{k} 
    = (V_{k-1}^{T} V_{k-m}^{T}) H_{k}^{0} (V_{k-m} V_{k-1}) 
    + \sum_{l=2}^{m} \rho_{k-l} (V_{k-1}^{T} \cdots V_{k-l+1}^{T}) s_{k-l} s_{k-l}^{T} (V_{k-l+1} \cdots V_{k-1})
    + \rho_{k-1} s_{k-1} s_{k-1}^{T}.
\]
Expanding the terms, we have for $j \in \{1, 2, \ldots, i\}$,
\[
    V_{k-i} \cdots V_{k-1} = I - \sum_{j=1}^{i} \rho_{k-j} y_{k-j} \tilde{v}_{k-j}^{T}
    \quad \text{where} \quad \tilde{v}_{k-j} = s_{k-j} - \sum_{l=1}^{j-1} (\rho_{k-l} y_{k-l}^{T} s_{k-j}) \tilde{v}_{k-l}.
\]
It follows that
\[
    H_{k} 
    = (I - \tilde{Y}\tilde{V}^{T})^{T} \gamma_{k} I (I - \tilde{Y}\tilde{V}^{T}) + \tilde{S} \tilde{S}^{T}
    = 
    \begin{bmatrix*}
        \sqrt{\gamma_k} (I - \tilde{Y}\tilde{V}^{T})^{T} & \tilde{S}
    \end{bmatrix*}
    \begin{bmatrix*}
        \sqrt{\gamma_k} (I - \tilde{Y}\tilde{V}^{T}) \\ 
        \tilde{S}^{T}.
    \end{bmatrix*}
    = \tilde{H}_k \tilde{H}_k^{T},
\]
where
\[
\begin{aligned}
    & \tilde{Y} = 
    \begin{bmatrix*} 
        \vert & & \vert \\
        \rho_{k-1} y_{k-1} & \cdots & \rho_{k-m} y_{k-m} \\
        \vert & & \vert \\
    \end{bmatrix*}, \\
    & \tilde{V} = 
    \begin{bmatrix*} 
        \vert & & \vert \\
        \tilde{v}_{k-1} & \cdots & \tilde{v}_{k-m} \\
        \vert & & \vert \\
    \end{bmatrix*}, \\
    & \tilde{S} = 
    \begin{bmatrix*} 
        \vert & & \vert \\
        \tilde{s}_{k-1} & \cdots & \tilde{s}_{k-m} \\
        \vert & & \vert \\
    \end{bmatrix*},
    \quad \tilde{s}_{k-1} = \sqrt{\rho_{k-1}} s_{k-1}, 
    ~ \tilde{s}_{k-l} = \sqrt{\rho_{k-l}} (V_{k-1}^{T} \cdots V_{k-l+1}^{T}) s_{k-l} ~ \text{for} ~ 2 \leq l \leq m.
\end{aligned}
\]
We now apply \cref{alg-unrolling-lbfgs} to unroll the above recurrence relations to compute columns of $\tilde Y, \tilde S$ and $\tilde V$. 

\begin{algorithm}[H]
  \centering
  \caption{Unrolling the \lbfgs{} Update}
  \label{alg-unrolling-lbfgs}
  \begin{algorithmic}
  \INPUT{saved directions $\{y_i\}_{i=k-1}^{k-m}$, saved steps $\{s_i\}_{i=k-1}^{k-m}$, saved inverse of inner products $\{\rho_i\}_{i=k-1}^{k-m}$}
    \STATE {$\tilde{y}_{k-1} = \rho_{k-1} y_{k-1}$}
    \STATE {$\tilde{v}_{k-1} = s_{k-1}$}
    \STATE {$\tilde{s}_{k-1} = \sqrt{\rho_{k-1}} s_{k-1}$}
    \FOR{$i = k-2, \dots, k-m$}
        \STATE {$\tilde{y}_i = \rho_i y_i$}
        \STATE {Set $\alpha = 0$}
        \FOR{$j = k-1, \dots, i+1$}
          \STATE {$\alpha = \alpha + (\tilde{y}_j^{T} s_i) \tilde{v}_j$}
        \ENDFOR
        \STATE {$\tilde{v}_i = s_i - \alpha$}
        \STATE {$\tilde{s}_i = \sqrt{\rho_i} (s_i - \alpha)$}
    \ENDFOR
  \OUTPUT{vectors $\{\tilde{y}_i, \tilde{v}_i, \tilde{s}_i\}_{i=k-1}^{k-m}$}
  \end{algorithmic}
\end{algorithm}

Since (non-zero) eigenvalues of $\tilde{H}_{k}^{T} H_L(w)\tilde{H}_{k}$ equal the eigenvalues of the preconditioned Hessian $H_{k} H_L(w) = \tilde{H}_k \tilde{H}_k^{T} H_L(w)$ (Theorem 1.3.22 of \citet{horn2012matrix}), we can analyze the spectrum of $\tilde{H}_{k}^{T}H_L(w)\tilde{H}_{k}$ instead.
This is advantageous since methods for calculating the spectral density of neural network Hessians are only compatible with symmetric matrices.

Since $\tilde{H}_{k}^{T} H_L(w)\tilde{H}_{k}$ is symmetric, we can use stochastic Lanczos quadrature (SLQ) \cite{golub2009matrices,lin2016approximating} to compute spectral density of this matrix.
SLQ only requires matrix-vector products with $\tilde H_k$ and Hessian-vector products, the latter of which may be efficiently computed via automatic differentiation; this is precisely what PyHessian does to compute spectral densities \cite{yao2020pyhessian}.

\begin{algorithm}[H]
  \centering
  \caption{Performing matrix-vector product}
  \label{alg-mvp}
  \begin{algorithmic}
  \INPUT{matrices $\tilde{Y}$, $\tilde{V}$, $\tilde{S}$ formed from resulting vectors from unrolling, vector $v$, and saved scaling factor for initializing diagonal matrix $\gamma_k$}
    \STATE {Split vector $v$ of length $\mathrm{size}(w) + m$ into $v_1$ of size $\mathrm{size}(w)$ and $v_2$ of size $m$}
    \STATE {$v' = \sqrt{\gamma_k}(v_1 - \tilde{V}\tilde{Y}^{T}v_1) + \tilde{S} v_2$}
    \STATE {Perform Hessian-vector-product on $v'$, and obtain $v''$}
    \STATE {Stack $\sqrt{\gamma_k}(v'' - \tilde{Y}\tilde{V}^{T}v'')$ and $\tilde{S}^{T}v''$, and obtain $v'''$}
  \OUTPUT{resulting vector $v'''$}
  \end{algorithmic}
\end{algorithm}

By combining the matrix-vector product procedure described in \cref{alg-mvp} with the Hessian-vector product operation, we are able to obtain spectral information of the preconditioned Hessian. 

\begin{figure*}
    \centering
    \includegraphics[trim={0 1.25cm 0 0}, clip, scale=0.45]{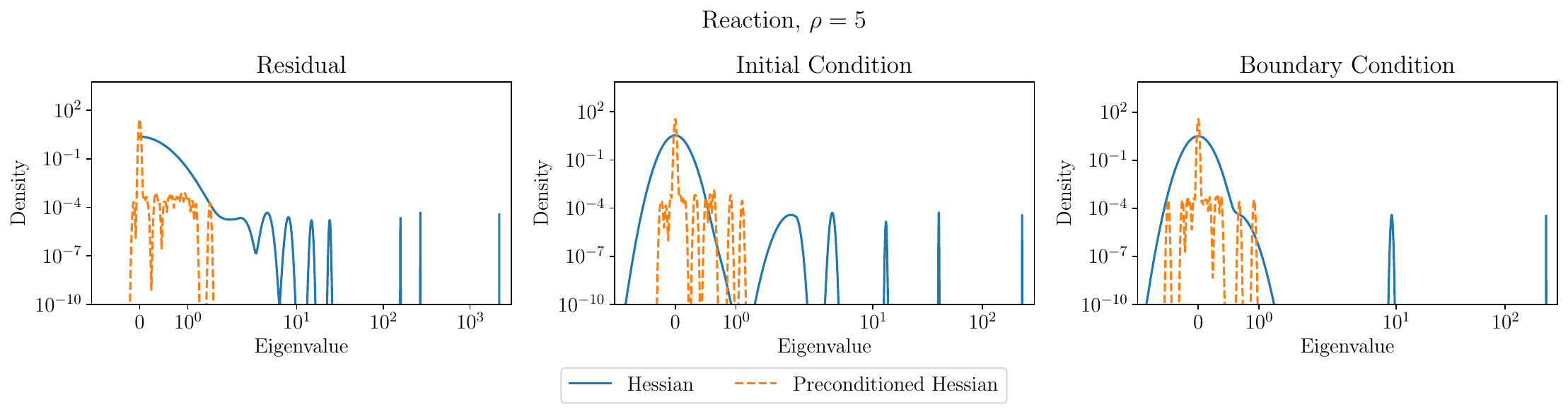}
    \includegraphics[scale=0.45]{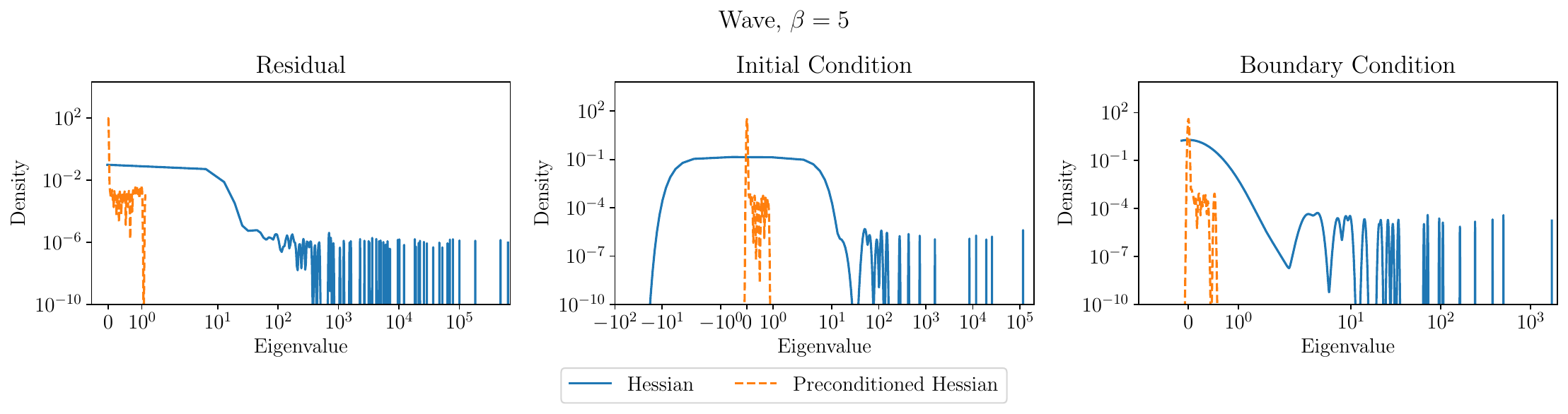}
    \caption{Spectral density of the Hessian and the preconditioned Hessian of each loss component after 41000 iterations of \al{} for the reaction and wave problems. The plots show the loss landscape of each component is ill-conditioned, and the conditioning of each loss component is improved by \lbfgs{}.}
    \label{fig:spectral_density_reaction_wave}
\end{figure*}
\section{\al{} Generally Gives the Best Performance}
\label{sec:opt_comparison_additional}
\cref{fig:opt_comparison} shows that \al{} typically yields the best performance on both loss and L2RE across network widths. 
\begin{figure*}
    \centering
    \includegraphics[scale=0.4]{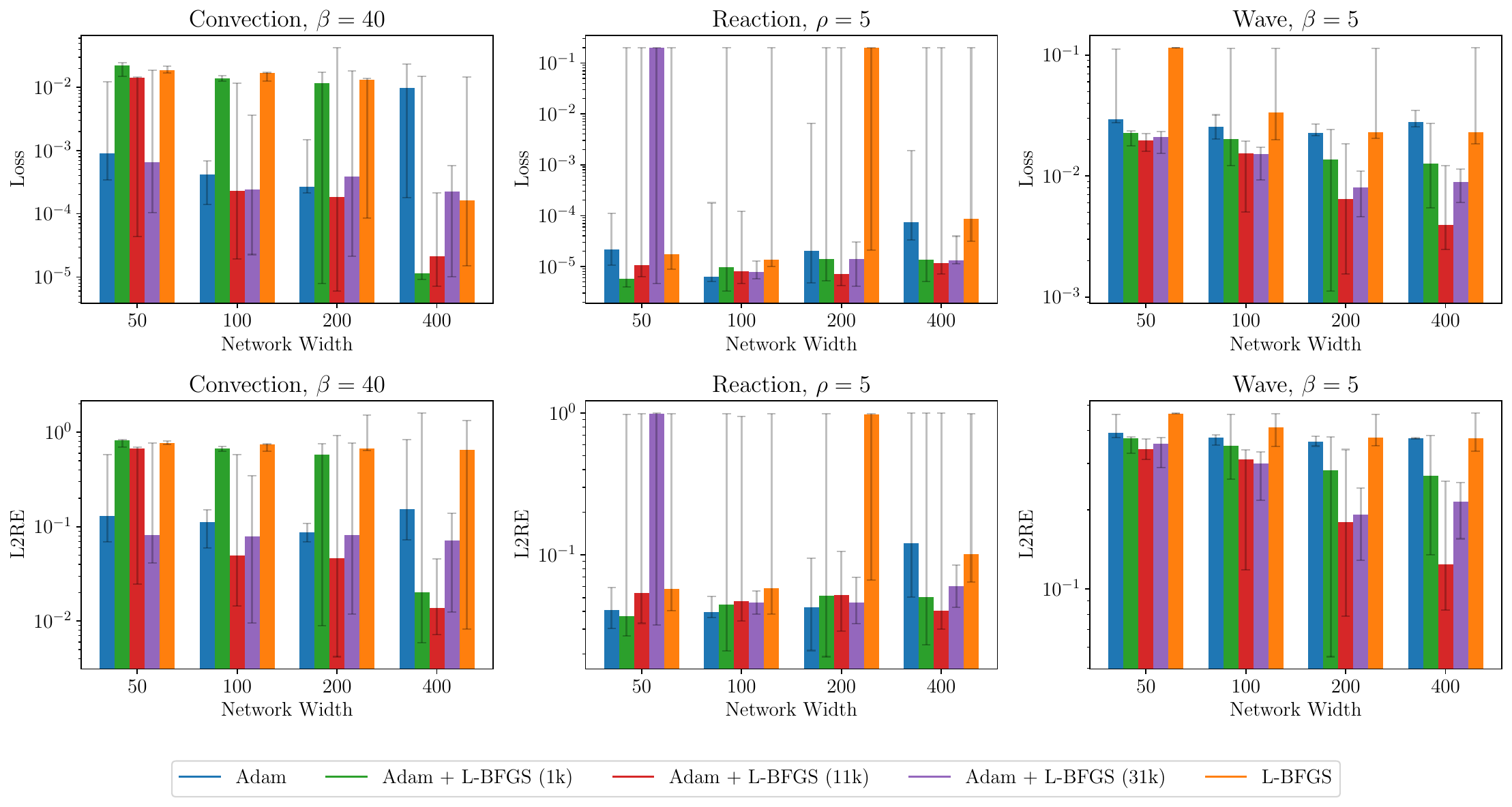}
    \caption{Performance of Adam, \lbfgs, and \al{} after tuning. 
    We find the learning rate $\eta^\star$ for each network width and optimization strategy that attains the lowest loss (L2RE) across all random seeds.
    The min, median, and max loss (L2RE) are calculated by taking the min, median, and max of the losses (L2REs) for learning rate $\eta^*$ across all random seeds.
    Each bar on the plot corresponds to the median, while the top and bottom error bars correspond to the max and min, respectively.
    The smallest min loss and L2RE are always attained by one of the \al{} strategies; the smallest median loss and L2RE are nearly always attained by one of the \al{} strategies.}
    \label{fig:opt_comparison}
\end{figure*}
\section{Additional details on Under-optimization}
\label{sec:under_optimization_additional}

\subsection{Early Termination of \lbfgs{}}
\cref{fig:line_search_multi_pde} explains why \lbfgs{} terminates early for the convection, reaction, and wave problems.
We evaluate the loss at $10^{4}$ uniformly spaced points in the interval $[0, 1]$.
The orange stars in \cref{fig:line_search_multi_pde} are step sizes that satisfy the strong Wolfe conditions and the red dots are step sizes that \lbfgs{} examines during the line search.

\begin{figure}
    \centering
    \includegraphics[scale=0.45]{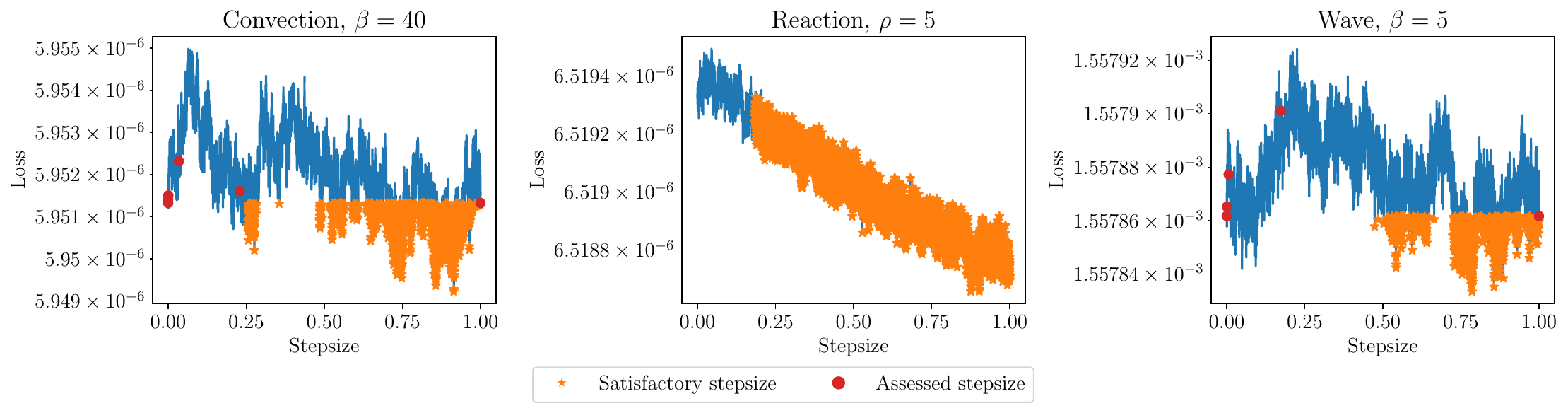}
    \caption{Loss evaluated along the \lbfgs{} search direction at different stepsizes after 41000 iterations of \al{}. For convection and wave, the line search does not find a stepsize that satisfies the strong Wolfe conditions, even though there are plenty of such points. For reaction, the slope of the objective used in the line search procedure at the current iterate is less than a pre-defined threshold $10^{-9}$, so \lbfgs{} terminates without performing any line-search.}
    \label{fig:line_search_multi_pde}
\end{figure}

\subsection{NysNewton-CG (NNCG)}
Here we present the NNCG algorithm (\cref{alg-NNCG}) introduced in \cref{subsec:NNCG} and its associated subroutines RandomizedNystr{\"o}mApproximation (\cref{alg-RNA}), Nystr\"{o}mPCG (\cref{alg-nyspcg}), and Armijo (\cref{alg-armijo}).
At each iteration, NNCG first checks whether the Nystr\"{o}m preconditioner (stored in $U$ and  $\hat{\Lambda}$) for the Nystr\"{o}mPCG method needs to be updated.
If so, the preconditioner is recomputed using the RandomizedNystr{\"o}mApproximation subroutine.
From here, the Newton step $d_k$ is computed using Nystr\"{o}mPCG; we warm start the PCG algorithm using the Newton step $d_{k - 1}$ from the previous iteration.
After computing the Newton step, we compute the step size $\eta_k$ using Armijo line search --- this guarantees that the loss will decrease when we update the parameters.
Finally, we update the parameters using $\eta_k$ and $d_k$.

In our experiments, we set $\eta = 1, K = 2000, s = 60, F = 20, \epsilon = 10^{-16}, M = 1000, \alpha = 0.1$, and $\beta = 0.5$.
We tune $\mu \in [10^{-5}, 10^{-4}, 10^{-3}, 10^{-2}, 10^{-1}]$; we find that $\mu = 10^{-2}, 10^{-1}$ work best in practice.
\cref{fig:under_optimization_intro,fig:under_optimization} show the NNCG run that attains the lowest loss after tuning $\mu$.

\begin{algorithm}[H]
	\centering
	\caption{NysNewton-CG (NNCG)}
	\label{alg-NNCG}
	\begin{algorithmic}
	\INPUT{Initialization $w_0$, max. learning rate $\eta$, number of iterations $K$, preconditioner sketch size $s$, preconditioner update frequency $F$, damping parameter $\mu$, CG tolerance $\epsilon$, CG max. iterations $M$, backtracking parameters $\alpha, \beta$}
    \STATE{$d_{-1} = 0$}
    \FOR{$k = 0, \dots, K - 1$}
        \IF{$k$ is a multiple of $F$} 
            \STATE{$[U, \hat{\Lambda}] = \textrm{RandomizedNystr{\"o}mApproximation}(H_{L}(w_k), s)$} \COMMENT{Update Nystr{\"o}m preconditioner every $F$ iterations}
        \ENDIF
        \STATE{$d_k = \textrm{Nystr{\"o}mPCG}(H_{L}(w_k), \nabla L(w_k), d_{k - 1}, U, \hat{\Lambda}, s, \mu, \epsilon, M)$} \COMMENT{Damped Newton step $(H_L(w_k) + \mu I)^{-1} \nabla L(w_k)$}
        \STATE{$\eta_k = \textrm{Armijo}(L, w_k, \nabla L(w_k), -d_k, \eta)$} \COMMENT{Compute step size via line search}
        \STATE{$w_{k+1} = w_k - \eta_k d_k$} \COMMENT{Update parameters}
    \ENDFOR
	\end{algorithmic}
\end{algorithm}

The RandomizedNystr{\"o}mApproximation subroutine (\cref{alg-RNA}) is used in NNCG to compute the preconditioner for Nystr\"{o}mPCG.
The algorithm returns the top-$s$ approximate eigenvectors and eigenvalues of the input matrix $M$.
Within NNCG, the sketch computation $Y = MQ$ is implemented using Hessian-vector products.
The portion in red is a fail-safe that allows for the preconditioner to be computed when $H$ is an indefinite matrix.
For further details, please see \citet{frangella2023randomized}.

\begin{algorithm}[H] 
   \caption{RandomizedNystr{\"o}mApproximation}
   \label{alg-RNA}
    \begin{algorithmic}
       \INPUT{Symmetric matrix $M$, sketch size $s$}
       \STATE{$S = \textrm{randn}(p, s)$} \COMMENT{Generate test matrix}
       \STATE{$Q = \textrm{qr\_econ}(S)$}
       \STATE{$Y = M Q$} \COMMENT{Compute sketch}
       \STATE $\nu = \sqrt{p} \text{eps}(\text{norm}(Y, 2))$ \hfill \COMMENT{Compute shift}
       \STATE $Y_{\nu} = Y + \nu Q$ \hfill \COMMENT{Add shift for stability}
       \STATE $\lambda = 0$ \hfill \COMMENT{Additional shift may be required for positive definiteness}
       \STATE $C = \text{chol}(Q^TY_\nu)$ \hfill \COMMENT{Cholesky decomposition: $C^{T}C = Q^{T}Y_\nu$}
       \textcolor{red}{
       \IF {chol fails}
        \STATE Compute $[W, \Gamma] = \mathrm{eig}(Q^T Y_\nu)$ \hfill \COMMENT{$Q^T Y_\nu$ is small and square}
        \STATE Set $\lambda = \lambda_{\min}(Q^T Y_\nu)$
        \STATE $R = W(\Gamma + |\lambda| I)^{-1/2} W^T$
        \STATE $B = YR$ \hfill \COMMENT{$R$ is psd}
       \ELSE
        \STATE $B = YC^{-1}$ \hfill \COMMENT{Triangular solve}
       \ENDIF
       }
       \STATE $[\hat V, \Sigma, \sim] = \text{svd}(B, 0)$ \hfill \COMMENT{Thin SVD}
       \STATE $\hat \Lambda = \text{max}\{0, \Sigma^2 - (\nu + |\lambda| I)\}$ \hfill \COMMENT{Compute eigs, and remove shift with element-wise max}
       \STATE {\bfseries Return:} $\hat V, \hat \Lambda$
    \end{algorithmic}
\end{algorithm} 

The Nystr\"{o}mPCG subroutine (\cref{alg-nyspcg}) is used in NNCG to compute the damped Newton step.
The preconditioner $P$ and its inverse $P^{-1}$ are given by 
\begin{align*}
    P &= \frac{1}{\hat{\lambda}_s + \mu} U (\hat{\Lambda} + \mu I) U^T + (I - UU^T) \\
    P^{-1} &= (\hat{\lambda}_s + \mu) U (\hat{\Lambda} + \mu I)^{-1} U^T + (I - UU^T).
\end{align*}
Within NNCG, the matrix-vector product involving the Hessian (i.e., $A = H_L(w_k)$) is implemented using Hessian-vector products.
For further details, please see \citet{frangella2023randomized}.

\begin{algorithm}[H] 
   \caption{Nystr\"{o}mPCG}
   \label{alg-nyspcg}
    \begin{algorithmic}
       \INPUT{Psd matrix $A$, right-hand side $b$, initial guess $x_0$, approx. eigenvectors $U$, approx. eigenvalues $\hat{\Lambda}$, sketch size $s$, damping parameter $\mu$, CG tolerance $\epsilon$, CG max. iterations $M$}
       \STATE{$r_0 = b - (A + \mu I) x_0$}
       \STATE{$z_0 = P^{-1} r_0$}
       \STATE{$p_0 = z_0$}
       \STATE{$k = 0$} \COMMENT{Iteration counter}
       \WHILE{$\|r_0\|_2 \geq \eps$ and $k < M$}
           \STATE{$v = (A + \mu I) p_0$}
           \STATE{$\alpha = (r_0^T z_0) / (p_0^T v_0)$} \COMMENT{Compute step size}
           \STATE{$x = x_0 + \alpha p_0$} \COMMENT{Update solution}
           \STATE{$r = r_0 - \alpha v$} \COMMENT{Update residual}
           \STATE{$z = P^{-1}r$}
           \STATE{$\beta = (r^T z) / (r_0^T z_0)$}
           \STATE{$x_0 \gets x, r_0 \gets r, p_0 \gets z + \beta p_0, z_0 \gets z, k \gets k + 1$}
       \ENDWHILE
       \STATE {\bfseries Return:} $x$
    \end{algorithmic}
\end{algorithm} 

The Armijo subroutine (\cref{alg-armijo}) is used in NNCG to guarantee that the loss decreases at every iteration.
The function oracle is implemented in PyTorch using a \textit{closure}.
At each iteration, the subroutine checks whether the \textit{sufficient decrease condition} has been met; if not, it shrinks the step size by a factor of $\beta$.
For further details, please see \citet{nocedal2006numerical}.

\begin{algorithm}[H] 
   \caption{Armijo}
   \label{alg-armijo}
    \begin{algorithmic}
       \INPUT{Function oracle $f$, current iterate $x$, current gradient $\nabla f(x)$, search direction $d$, initial step size $t$, backtracking parameters $\alpha, \beta$}
       \WHILE{$f(x + td) > f(x) + \alpha t (\nabla f(x)^T d)$} 
        \STATE{$t \gets \beta t$} \COMMENT{Shrink step size}
       \ENDWHILE
       \STATE {\bfseries Return:} $t$
    \end{algorithmic}
\end{algorithm} 

\subsection{Wall-clock Times for \lbfgs{} and NNCG}

\begin{table}[t]
    \caption{Per-iteration times (in seconds) of \lbfgs{} and NNCG on each PDE.}
    \vskip 0.15in
    \centering
    \begin{tabular}{|c|c|c|c|}
        \hline
        Optimizer & Convection & Reaction & Wave \\
        \hline
        \lbfgs{} & 4.6e-2 & 3.6e-2 & 9.0e-2 \\
        \hline
        NNCG & 2.5e-1 & 7.2e-1 & 2.9e1 \\
        \hline
        Time Ratio & 5.43 & 20 & 322.22 \\
        \hline
    \end{tabular}
    \label{tab:wall_clock_time_comparison}
\end{table}

\cref{tab:wall_clock_time_comparison} summarizes the per-iteration wall-clock times of \lbfgs{} and NNCG on each PDE. The large gap on wave (compared to reaction and convection) is because NNCG has to compute hessian-vector products involving second derivatives, while this is not the case for the two other PDEs. 


\section{Ill-conditioned Differential Operators Lead to Difficult Optimization Problems}
In this section, we state and prove the formal version of \cref{thm:informal_ill_cond}.
The overall structure of the proof is based on showing the conditioning of the Gauss-Newton matrix of the population PINN loss is controlled by the conditioning of the differential operator.
We then show the empirical Gauss-Newton matrix is close to its population counterpart by using matrix concentration techniques. 
Finally, as the conditioning of $H_L$ at a minimizer is controlled by the empirical Gauss-Newton matrix, we obtain the desired result. 

\label{section:ill-cond-D}
\subsection{Preliminaries}
Similar to \citet{de2023operator}, we consider a general linear PDE with Dirichlet boundary conditions:
\[
\begin{array}{ll}
    & \Dc[u](x) = f(x),\quad x\in \Omega, \\
    & u(x) = g(x), \quad x\in \partial \Omega,
\end{array}
\]
where $u: \R^d \mapsto \R$, $f:\R^d \mapsto \R$ and $\Omega$ is a bounded subset of $\R^d$.
The ``population'' PINN objective for this PDE is
\[
L_\infty(w) = \frac{1}{2}\int_{\Omega}\left(\Dc[u(x;w)]-f(x)\right)^2d\mu(x)+\frac{\lambda}{2} \int_{\partial \Omega}\left(u(x; w)-g(x)\right)^2d\sigma(x).
\]
$\lambda$ can be any positive real number; we set $\lambda = 1$ in our experiments.
Here $\mu$ and $\sigma$ are probability measures on $\Omega$ and $\partial \Omega$ respectively, from which the data is sampled. 
The empirical PINN objective is given by
\[
L(w) = \frac{1}{2\nres}\sum_{i=1}^{\nres}\left(\Dc[u(x^i_r;w)]-f(x_i)\right)^2+\frac{\lambda}{2\nbc}\sum_{j=1}^{\nbc}\left(u(x_b^j;w)-g(x_j)\right)^2.
\]
Moreover, throughout this section we use the notation $\langle f,g\rangle_{L^{2}(\Omega)}$ to denote the standard $L^2$-inner product on $\Omega$:
\[
\langle f,g\rangle_{L^{2}(\Omega)} = \int_{\Omega}fg d\mu(x).
\]

\begin{lemma}
    The Hessian of the $L_\infty(w)$ is given by
    \begin{align*}
    H_{L_\infty}(w) & = \int_{\Omega}\Dc[\nabla_w u(x;w)]\Dc[\nabla_w u(x;w)]^{T}d\mu(x)+\int_{\Omega}\Dc[\nabla^2_w u(x;w)]\left(\Dc[\nabla_w u(x;w)]-f(x)\right)d\mu(x)\\
    & + \lambda\int_{\partial \Omega}\nabla_w u(x; w)\nabla_w u(x; w)^{T}d\sigma(x) + \lambda\int_{\partial \Omega}\nabla^2_w u(x; w)\left(u(x;w)-g(x)\right)d\sigma(x). 
    \end{align*}
    The Hessian of $L(w)$ is given by
    \begin{align}
       H_L(w) & = \frac{1}{n_{\textup{res}}}\sum^{n_{\textup{res}}}_{i=1}\Dc[\nabla_w u(x_r^i; w)]\Dc[\nabla_w u(x_r^i; w)]^{T}+\frac{1}{n_{\textup{res}}}\sum^{n_{\textup{res}}}_{i=1}\Dc[\nabla^2_w u(x^r_i;w)]\left(\Dc[\nabla_w u(x_r^i;w)]-f(x_r^i)\right)\\
       & +\frac{\lambda}{n_{\textup{bc}}}\sum_{j=1}^{n_{\textup{bc}}}\nabla_w u(x_b^j;w)\nabla_w u(x_b^j;w)^{T} + \frac{\lambda}{n_{\textup{bc}}}\sum^{n_{\textup{bc}}}_{j=1}\nabla^2_w u(x_b^j;w)\left(u(x_b^j;w)-g(x_j)\right). \nonumber
    \end{align}
    In particular, for $w_\star\in \Wstar$
    \begin{align*}
        H_L(w_\star) = G_r(w)+ G_b(w).
    \end{align*}
    Here 
    \[
    G_r(w) \coloneqq \frac{1}{n_{\textup{res}}}\sum^{n_{\textup{res}}}_{i=1}\Dc[\nabla_w u(x_i; w_\star)]\Dc[\nabla_w u(x_i; w_\star)]^{T},\quad G_b(w) = \frac{\lambda}{n_{\textup{bc}}}\sum_{j=1}^{n_{\textup{bc}}}\nabla_w u(x_b^j;w_\star)\nabla_w u(x_b^j;w_\star)^{T}.
    \]
\end{lemma}
Define the maps $\F_{\textup{res}}(w) = \begin{bmatrix}
    \Dc[u(x_r^1;w)] \\
    \vdots \\
    \Dc[u(x_r^{\nres};w)]
\end{bmatrix}$,
and $\F_{\textup{bc}}(w) = \begin{bmatrix}
    u(x_b^1;w) \\
    \vdots \\
    u(x_b^{\nbc};w)]
\end{bmatrix}$.
We have the following important lemma, which follows via routine calculation. 
\begin{lemma}
\label{lemma:jac_ntk}
    Let $n = \nres+\nbc$. Define the map $\mathcal F:\R^p\rightarrow \R^{n}$, by stacking $\F_{\textup{res}}(w), \F_{\textup{bc}}(w)$.
    Then, the Jacobian of $\F$ is given by
    \[
    J_\F(w) = \begin{bmatrix}
        J_{\F_{\textup{res}}}(w) \\
        J_{\F_{\textup{bc}}}(w).
    \end{bmatrix}
    \]
    Moreover, the tangent kernel $K_\F(w) = J_\F(w)J_\F(w)^{T}$ is given by
    \[ K_\F(w) = 
    \begin{bmatrix}
        J_{\F_{\textup{res}}}(w)J_{\F_{\textup{res}}}(w)^{T} & J_{\F_{\textup{res}}}(w)J_{\F_{\textup{bc}}}(w)^{T}  \\
        J_{\F_{\textup{bc}}}(w)J_{\F_{\textup{res}}}(w)^{T} & J_{\F_{\textup{bc}}}(w)J_{\F_{\textup{bc}}}(w)^{T} 
    \end{bmatrix} =
    \begin{bmatrix}
        K_{\F_{\textup{res}}}(w) & J_{\F_{\textup{res}}}(w)J_{\F_{\textup{bc}}}(w)^{T}  \\
        J_{\F_{\textup{bc}}}(w)J_{\F_{\textup{res}}}(w)^{T} & K_{\F_{\textup{bc}}}(w) 
    \end{bmatrix}.
    \]
\end{lemma}

\subsection{Relating $G_{\infty}(w)$ to $\mathcal D$}
Isolate the population Gauss-Newton matrix for the residual:
\[
G_{\infty}(w) = \int_{\Omega}\Dc[\nabla_w u(x;w)]\Dc[\nabla_w u(x;w)]^{T}d\mu(x).
\]
Analogous to \citet{de2023operator} we define the functions $\phi_i(x;w) = \partial_{w_i}u(x;w)$ for $i\in\{1\dots,p\}$.
From this and the definition of $G_{\infty}(w)$, it follows that $\left(G_\infty(w)\right)_{ij} = \langle \Dc[\phi_i], \Dc[\phi_j]\rangle_{L^2(\Omega)}$.

Similar to \citet{de2023operator} we can associate each $w\in\R^p$ with a space of functions $\mathcal H(w) = \textup{span}\left(\phi_1(x;w),\dots,\phi_p(x;w)\right)\subset L^2(\Omega).$
We also define two linear maps associated with $\mathcal H(w)$:
\[
T(w)v = \sum_{i=1}^{p}v_i\phi_i(x;w),
\]
\[
T^{*}(w)f = \left(\langle f,\phi_1\rangle_{L^2(\Omega)},\dots,\langle f,\phi_p\rangle_{L^2(\Omega)}\right).
\]
From these definitions, we establish the following lemma. 
\begin{lemma}[Characterizing $G_{\infty}(w)$]
\label{lemma:Pop-GN}
Define $\mathcal A = \Dc^{*}\Dc$. 
Then the matrix $G_{\infty}(w)$ satisfies
    \[
    G_{\infty}(w) = T^{*}(w)\mathcal A T(w). 
    \]
\end{lemma}
\begin{proof}
Let $e_i$ and $e_j$ denote the $ith$ and $jth$ standard basis vectors in $\R^p$. 
Then,
\begin{align*}
(G_{\infty}(w))_{ij} &= \langle \Dc[\phi_i](w), \Dc[\phi_j](w)\rangle_{L^2(\Omega)} = \langle \phi_i(w),\Dc^{*}\Dc[\phi_j(w)] \rangle_{L^2(\Omega)} = \langle Te_i, \Dc^{*}\Dc[Te_j]\rangle_{L^2(\Omega)} \\
&= \langle e_i, (T^{*}\Dc^{*}\Dc T)[e_j]\rangle_{L^2(\Omega)},
\end{align*}
where the second equality follows from the definition of the adjoint. 
Hence, using $\mathcal A = \Dc^{*}\Dc$, we conclude $G_{\infty}(w) = T^{*}(w)\mathcal A T(w)$.
\end{proof}

Define the kernel integral operator $\mathcal K_{\infty}(w):L^2(\Omega)\rightarrow \mathcal H$ by
\begin{equation}
\label{eq:kern_op}
\mathcal K_{\infty}(w)[f](x) = T(w)T^{*}(w)f = \sum_{i=1}^{p}\langle f,\phi_i(x;w)\rangle \phi_i(x;w),
\end{equation}
and the kernel matrix $A(w)$ with entries $A_{ij}(w) = \langle \phi_{i}(x;w),\phi_{j}(x;w)\rangle_{L^2(\Omega)}$. 

Using \cref{lemma:Pop-GN} and applying the same logic as in the proof of Theorem 2.4 in \citet{de2023operator},
we obtain the following theorem. 
\begin{theorem}
\label{thm:pop-gn-eigvals}
    Suppose that the matrix $A(w)$ is invertible.
    Then the eigenvalues of $G_{\infty}(w)$ satisfy
    \[
    \lambda_j(G_\infty(w)) = \lambda_j(\mathcal A\circ \Kc_\infty(w)),\quad \text{for all}~j\in[p].
    \]
\end{theorem}


\subsection{$G_r(w)$ Concentrates Around $G_{\infty}(w)$}
In order to relate the conditioning of the population objective to the empirical objective, we must relate the population Gauss-Newton residual matrix to its empirical counterpart. 
We accomplish this by showing $G_r(w)$ concentrates around $G_{\infty}(w)$. 
To this end, we recall the following variant of the intrinsic dimension matrix Bernstein inequality from \citet{tropp2015introduction}.
\begin{theorem}[Intrinsic Dimension Matrix Bernstein]
 \label{thm:int_bern}
    Let $\{X_i\}_{i\in [n]}$ be a sequence of independent mean zero random matrices of the same size. 
    Suppose that the following conditions hold:
    \begin{align*}
        &\|X_i\|  \leq B,~\sum^{n}_{i=1}\mathbb E[X_i X_i^{T}]\preceq V_1,~\sum^{n}_{i=1}\mathbb E[X_i^{T} X_i]\preceq V_2.
    \end{align*}
    Define 
    \[\mathcal V = 
    \begin{bmatrix}
        V_1 & 0 \\
        0   &  V_2
    \end{bmatrix},~ \varsigma^2 = \max\{\|V_1\|,\|V_2\|\},
    \]
    and the \emph{intrinsic dimension} $d_{\textup{int}} = \frac{\textup{trace}(\mathcal V)}{\|\mathcal V\|}$.
    \newline
    Then for all $t\geq \varsigma+\frac{B}{3}$, 
    \[
    \mathbb P\left(\left\|\sum^{n}_{i=1}X_i\right\|\geq t\right) \leq 4d_{\textup{int}}\exp\left(-\frac{3}{8}\min\left\{\frac{t^2}{\varsigma^2},\frac{t}{B}\right\}\right).
    \]   
\end{theorem}

Next, we recall two key concepts from the kernel ridge regression literature and approximation via sampling literature: $\gamma$-\emph{effective dimension} and $\gamma$-\emph{ridge leverage coherence} \cite{bach2013sharp,cohen2017input,rudi2017falkon}. 
\begin{definition}[$\gamma$-Effective dimension and $\gamma$-ridge leverage coherence]
    Let $\gamma>0$. 
    Then the $\gamma$-effective dimension of $G_{\infty}(w)$ is given by
    \[
    d^{\gamma}_{\textup{eff}}(G_{\infty}(w)) = \textup{trace}\left(G_{\infty}(w)(G_{\infty}(w)+\gamma I)^{-1}\right).
    \]
    The $\gamma$-ridge leverage coherence is given by
    \[
    \chi^\gamma(G_{\infty}(w)) = \sup_{x\in \Omega}\frac{\left\|(G_{\infty}(w)+\gamma I)^{-1/2}\Dc[\nabla_w u(x;w)]\right\|^2}{\mathbb E_{x\sim \mu}\left\|(G_{\infty}(w)+\gamma I)^{-1/2}\Dc[\nabla_w u(x;w)]\right\|^2} = \frac{\sup_{x\in \Omega}{\left\|(G_{\infty}(w)+\gamma I)^{-1/2}\Dc[\nabla_w u(x;w)]\right\|^2}}{{d^{\gamma}_{\textup{eff}}(G_{\infty}(w))}}.
    \]
\end{definition}
Observe that $d^{\gamma}_{\textup{eff}}(G_{\infty}(w))$ only depends upon $\gamma$ and $w$, while $\chi^\gamma(G_{\infty}(w)) $ only depends upon $\gamma, w,$ and $\Omega$. 
Moreover, $\chi^\gamma(G_{\infty}(w))<\infty$ as $\Omega$ is bounded. 

We prove the following lemma using the $\gamma$-effective dimension and $\gamma$-ridge leverage coherence in conjunction with \cref{thm:int_bern}.
\begin{lemma}[Finite-sample approximation]
\label{lemma:sampling}
Let $0<\gamma<\lambda_1(G_{\infty}(w))$. 
If $\nres\geq 40\chi^\gamma(G_{\infty}(w))d^{\gamma}_{\textup{eff}}(G_{\infty}(w))\log\left(\frac{8d^{\gamma}_{\textup{eff}}(G_{\infty}(w))}{\delta}\right)$, then with probability at least $1-\delta$
    \[
    \frac{1}{2}\left[G_\infty(w)-\gamma I\right] \preceq G_{r}(w)\preceq \frac{1}{2}\left[3 G_\infty(w)+\gamma I.\right]
    \]
\end{lemma}
\begin{proof}
    Let $x_i = (G_{\infty}(w)+\gamma I)^{-1/2}\Dc[\nabla_w u(x_i;w)]$, and $X_i = \frac{1}{\nres}\left(x_ix_i^{T}-D_\gamma\right)$, where $D_\gamma = G_{\infty}(w)\left(G_{\infty}(w)+\gamma I\right)^{-1}$.
    Clearly, $\mathbb E[X_i] = 0$. 
    Moreover, the $X_i$'s are bounded as
    \begin{align*}
    \|X_i\| & = \max\left\{\frac{\lamMax(X_i)}{\nres},-\frac{\lamMin(X_i)}{\nres}\right\} \leq \max\left\{\frac{\|x_i\|^2}{\nres}, \frac{\lamMax(-X_i)}{\nres}\right\} \leq \max\left\{\frac{\chi^{\gamma}(G_{\infty}(w))d^{\gamma}_{\textup{eff}}(G_{\infty}(w))}{\nres},\frac{1}{\nres}\right\} \\
    & = \frac{\chi^{\gamma}(G_{\infty}(w))d^{\gamma}_{\textup{eff}}(G_{\infty}(w))}{\nres}.
    \end{align*}
    Thus, it remains to verify the variance condition. 
    We have
    \begin{align*}
    \sum^{\nres}_{i=1}\mathbb E[X_iX_i^{T}] & = \nres\mathbb E[X_1^2] = \nres\times \frac{1}{\nres^2}\mathbb E[(x_1x_1^{T}-D_\gamma)^{2}]\preceq \frac{1}{\nres}\mathbb E[\|x_1\|^2 x_1 x_1^{T}] \\ 
    & \preceq \frac{\chi^{\gamma}(G_{\infty}(w))d^{\gamma}_{\textup{eff}}(G_{\infty}(w))}{\nres} D_\gamma. 
    \end{align*}
    Hence, the conditions of \cref{thm:int_bern} hold with $B = \frac{\chi^{\gamma}(G_{\infty}(w))d^{\gamma}_{\textup{eff}}(G_{\infty}(w))}{\nres}$ and $V_1 = V_2 = \frac{\chi^{\gamma}(G_{\infty}(w))d^{\gamma}_{\textup{eff}}(G_{\infty}(w))}{\nres} D_\gamma$.
    Now $1/2 \leq \|\mathcal V\|\leq 1$ as $\nres\geq \chi^{\gamma}(G_{\infty}(w))d^{\gamma}_{\textup{eff}}(G_{\infty}(w))$ and $\gamma\leq \lambda_1\left(G_\infty(w)\right)$.
    Moreover, as $V_1 = V_2$ we have $d_{\textup{int}} \leq 4 d^{\gamma}_{\textup{eff}}(G_{\infty}(w))$. 
    So, setting 
    \[
    t = \sqrt{\frac{8\chi^\gamma(G_{\infty}(w))d^{\gamma}_{\textup{eff}}(G_{\infty}(w))\log\left(\frac{8d^{\gamma}_{\textup{eff}}(G_{\infty}(w))}{\delta}\right)}{3\nres}}+\frac{8\chi^\gamma(G_{\infty}(w))d^{\gamma}_{\textup{eff}}(G_{\infty}(w))\log\left(\frac{8d^{\gamma}_{\textup{eff}}(G_{\infty}(w))}{\delta}\right)}{3\nres}
    \]
    and using $\nres\geq 40\chi^\gamma(G_{\infty}(w)) d^{\gamma}_{\textup{eff}}(G_{\infty}(w))\log\left(\frac{8d^{\gamma}_{\textup{eff}}(G_{\infty}(w))}{\delta}\right)$, we conclude
    \[\mathbb P\left(\left\|\sum_{i=1}^{\nres}X_i\right\|\geq \frac{1}{2}\right)\leq \delta.\]
    Now, $\left\|\sum_{i=1}^{\nres}X_i\right\|\leq \frac{1}{2}$ implies
    \[
    -\frac{1}{2}\left[G_{\infty}(w)+\gamma I\right]\preceq G_r(w)-G_{\infty}(w)\preceq \frac{1}{2}\left[G_{\infty}(w)+\gamma I\right].
    \]
    The claim now follows by rearrangement. 
\end{proof}

By combining \cref{thm:pop-gn-eigvals} and \cref{lemma:sampling}, we show that if the spectrum of $\A\circ \Kc_{\infty}(w)$ decays, then the spectrum of the empirical Gauss-Newton matrix also decays with high probability.  
\begin{proposition}[Spectrum of empirical Gauss-Newton matrix decays fast]
\label{prop:emp_gn_spectrum}
Suppose the eigenvalues of $\A\circ \Kc_{\infty}(w)$ satisfy $\lambda_j(\mathcal A\circ \Kc_{\infty}(w))\leq Cj^{-2\alpha}$, where $\alpha>1/2$ and $C>0$ is some absolute constant.
Then if $\sqrt{\nres}\geq 40C_1\chi^{\gamma}(G_{\infty}(w))\log\left(\frac{1}{\delta}\right)$, for some absolute constant $C_1$, it holds that
\[
  \lambda_{\nres}(G_r(w))\leq \nres^{-\alpha}
\]
with probability at least $1-\delta$.
         
\end{proposition}
\begin{proof} 
    The hypotheses on the decay of the eigenvalues implies $d^{\gamma}_{\textup{eff}}(G_{\infty}(w)) \leq C_1\gamma^{-\frac{1}{2\alpha}}$ (see Appendix C of \citet{bach2013sharp}).
    Consequently, given $\gamma = \nres^{-\alpha}$, we have $d^{\gamma}_{\textup{eff}}(G_{\infty}(w)) \leq C_1\nres^{\frac{1}{2}}$. 
    Combining this with our hypotheses on $\nres$, it follows $\nres\geq 40 C_1\chi^{\gamma}(G_{\infty}(w))d^{\gamma}_{\textup{eff}}(G_{\infty}(w))\log\left(\frac{8d^{\gamma}_{\textup{eff}}(G_{\infty}(w))}{\delta}\right)$.
    Hence \cref{lemma:sampling} implies with probability at least $1-\delta$ that 
    \[
    G_r(w)\preceq \frac{1}{2}\left(3 G_\infty(w)+\gamma I\right),
    \]
    which yields for any $1\leq r\leq n$
    \[
    \lambda_{\nres}(G_r(w))\leq \frac{1}{2}\left(3\lambda_r(G_\infty(w))+\gamma\right).
    \]
    Combining the last display with $\nres\geq 3d^{\gamma}_{\textup{eff}}(G_{\infty}(w))$, 
    Lemma 5.4 of \citet{frangella2023randomized} guarantees $\lambda_r(G_\infty(w))\leq \gamma/3$, and so 
    \[
    \lambda_{\nres}(G_r(w))\leq \frac{1}{2}\left(3\lambda_r(G_\infty(w))+\gamma\right)\leq \gamma \leq \nres^{-\alpha}.
    \]
\end{proof}

\subsection{Formal Statement of \cref{thm:informal_ill_cond} and Proof}
\begin{theorem}[An ill-conditioned differential operator leads to hard optimization]
    Fix $w_\star \in \Wstar$, and let $\mathcal S$ be a set containing $w_\star$ for which $\mathcal S$ is $\mu$-\PL.
    Let $\alpha>1/2$.
    If the eigenvalues of $\A\circ \Kc_{\infty}(w_\star)$ satisfy $\lambda_j(\mathcal A\circ \Kc_{\infty}(w_\star))\leq C j^{-2\alpha}$ and $\sqrt{\nres}\geq 40 C_1\chi^{\gamma}(G_{\infty}(w_\star))\log\left(\frac{1}{\delta}\right)$, then 
    \[
            \kappa_L(\mathcal S) \geq C_2\nres^{\alpha},
    \]
    with probability at least $1-\delta$.
    Here $C, C_1,$ and $C_2$ are absolute constants. 
        
\end{theorem}

\begin{proof}
    By the assumption on $\nres$, the conditions of \cref{prop:emp_gn_spectrum} are met, so, 
    \[
    \lambda_{\nres}(G_r(w_\star))\leq \nres^{-\alpha}.
    \] 
    with probability at least $1-\delta$.
    By definition $G_r(w_\star) = J_{\F_{\textup{res}}}(w_\star)^{T}J_{\F_{\textup{res}}}(w_\star)$, consequently,
    \[
    \lambda_{\nres}(K_{\F_{\textup{res}}}(w_\star)) = \lambda_{\nres}(G_r(w_\star)) \leq \nres^{-\alpha}.
    \] 
    Now, the \PL-constant for $\mathcal S$, satisfies $\mu = \inf_{w \in \mathcal S}\lambda_{n}(K_{\F}(w))$ \cite{liu2022loss}. 
    Combining this with the expression for $K_\F(w_\star)$ in \cref{lemma:jac_ntk}, we reach 
    \[
    \mu\leq \lambda_n(K_\F(w_\star))\leq \lambda_{\nres}(K_{\F_{\textup{res}}}(w_\star))\leq \nres^{-\alpha},
    \]
    where the second inequality follows from Cauchy's Interlacing theorem. 
    Recalling that $\kappa_L(\mathcal S) = \frac{\sup_{w \in \mathcal S}\|H_L(w)\|}{\mu}$, and $H_L(w_\star)$ is symmetric psd, we reach
    \begin{align*}
        \kappa_L(\mathcal S) \geq \frac{\lambda_1(H_L(w_\star))}{\mu}\overset{(1)}{\geq} \frac{\lambda_1(G_r(w_\star))+\lambda_p(G_b(w_\star))}{\mu} \overset{(2)}{=} \frac{\lambda_1(G_r(w_\star))}{\mu} \overset{(3)}{\geq} C_3\lambda_1(G_\infty(w_\star))\nres^{\alpha}. 
    \end{align*}
    Here $(1)$ uses $H_L(w_\star) = G_r(w_\star)+G_b(w_\star)$ and Weyl's inequalities, $(2)$ uses $p\geq \nres+\nbc$, so that $\lambda_p(G_b(w_\star)) = 0$.
    Inequality $(3)$ uses the upper bound on $\mu$ and the lower bound on $G_r(w)$ given in \cref{lemma:sampling}. 
    Hence, the claim follows with $C_2 = C_3\lambda_1(G_\infty(w_\star))$.
\end{proof}
\subsection{$\kappa$ Grows with the Number of Residual Points}
\label{subsec:kappa_grows}
\begin{figure*}
    \centering
    \includegraphics[scale=0.45]{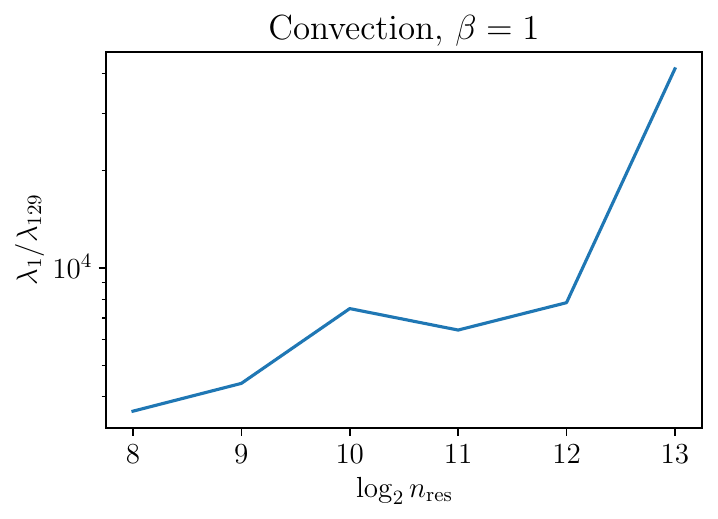}
    \caption{Estimated condition number after 41000 iterations of \al{} with different number of residual points from $255 \times 100$ grid on the interior. Here $\lambda_i$ denotes the $i$th largest eigenvalue of the Hessian. The model has $2$ layers and the hidden layer has width $32$. The plot shows $\kappa_L$ grows polynomially in the number of residual points.}
    \label{fig:condition_number_bound}
\end{figure*}
\Cref{fig:condition_number_bound} plots the ratio $\lambda_1(H_L)/\lambda_{129}(H_L)$ near a minimizer $w_\star$. This ratio is a lower bound for the condition number of $H_L$, and is computationally tractable to compute. 
We see that the estimate of the $\kappa$ grows polynomially with $\nres$, which provides empirical verification for \cref{thm:informal_ill_cond}.


\section{Convergence of GDND (\cref{alg-GDND})}
\label{section:GDND_conv}
In this section, we provide the formal version of \cref{thm:GDND_informal} and its proof. 
However, this is delayed till \cref{subsec:GDND_conv}, as the theorem is a consequence of a series of results.
Before jumping to the theorem, we recommend reading the statements in the preceding subsections to understand the statement and corresponding proof. 
\subsection{Overview and Notation}
Recall, we are interested in minimizing the objective in \eqref{eq:pinn_prob_gen}:
\[
L(w) = \frac{1}{2\nres}\sum_{i=1}^{\nres}\left(\Dc[u(x_r^i;w)]\right)^2+\frac{1}{2\nbc}\sum_{j=1}^{\nbc}\left(\Bc[u(x_b^j;w)]\right)^2, 
\]
where $\Dc$ is the differential operator defining the PDE and $\Bc$ is the operator defining the boundary conditions. 
Define
\[
\F(w) = \begin{bmatrix}
    \frac{1}{\sqrt{\nres}}\Dc[u(x^1_r;w)] \\
    \vdots \\
    \frac{1}{\sqrt{\nres}}\Dc[u(x_r^{\nres};w)]\\
    \frac{1}{\sqrt{\nbc}}\Bc[u(x^1_b;w)] \\
    \vdots \\
    \frac{1}{\sqrt{\nbc}}\Bc[u(x_b^{\nbc};w)]
\end{bmatrix},~ y = 0
\]
Using the preceding definitions, our objective may be rewritten as:
\[
L(w) = \frac{1}{2}\|\F(w)-y\|^2.
\]
Throughout the appendix, we work with the condensed expression for the loss given above.
We denote the $(\nres+\nbc)\times p$ Jacobian matrix of $\mathcal F$ by $J_\F(w)$. 
The tangent kernel at $w$ is given by the $n\times n$ matrix $K_\F(w) = J_\F(w) J_\F(w)^{T}$.
The closely related Gauss-Newton matrix is given by $G(w) = J_\F(w)^{T} J_\F(w)$.

\subsection{Global Behavior: Reaching a Small Ball About a Minimizer}
We begin by showing that under appropriate conditions, gradient descent outputs a point close to a minimizer after a fixed number of iterations.
We first start with the following assumption which is common in the neural network literature \cite{liu2022loss,liu2023aiming}.
\begin{assumption}
\label{assp:loss_reg}
     The mapping $\mathcal F(w)$ is $\mathcal L_\F$-Lipschitz, and the loss $L(w)$ is $\beta_{L}$-smooth.
\end{assumption}


Under \Cref{assp:loss_reg} and a \PL-condition, we have the following theorem of \citet{liu2022loss}, which shows gradient descent converges linearly. 
\begin{theorem}
\label{thm:grad_dsct_conv}
    Let $w_0$ denote the network weights at initialization. 
    Suppose \Cref{assp:loss_reg} holds, and that $L(w)$ is $\mu$-P\L$^{\star}$ in $B(w_0,2R)$ with $R = \frac{2\sqrt{2\beta_{L}L(w_0)}}{\mu}$.
    Then the following statements hold:
    \begin{enumerate}
        \item The intersection $B(w_0,R)\cap\Wstar$ is non-empty.
        \item Gradient descent with step size $\eta = 1/\beta_L$ satisfies:
        \begin{align*}
        &w_{k+1} = w_k-\eta \nabla L(w_k)\in B(w_0,R)~ \text{for all } k\geq 0,\\
        &L(w_k)\leq \left(1-\frac{\mu}{\beta_L}\right)^kL(w_0).
        \end{align*}
    \end{enumerate}
\end{theorem}
For wide neural neural networks, it is known that the $\mu$-\PL condition in \cref{thm:grad_dsct_conv} hold with high probability, see \citet{liu2022loss} for details.

We also recall the following lemma from \citet{liu2023aiming}.
    \begin{lemma}[Descent Principle]
    \label{lemma:descent_principle}
        Let $L:\R^p\mapsto [0,\infty)$ be differentiable and $\mu$-\PL in the ball $B(w,r)$. 
        Suppose $L(w)<\frac{1}{2}\mu r^2$.
        Then the intersection $B(w,r)\cap \Wstar$ is non-empty, and
        \[
        \frac{\mu}{2}\dist^2(w,\Wstar)\leq L(w).
        \]
    \end{lemma}
        Let $\Lc_{H_L}$ be the Hessian Lipschitz constant in $B(w_0,2R)$, and $\Lc_{J_\F} = \sup_{w\in B(w_0,2R)}\|H_\F(w)\|$, where $\|H_\F(w)\| = \max_{i\in [n]}\|H_{\F_i}(w)\|$. 
        Define $M = \max\{\mathcal L_{\HL},\Lc_{J_\F},\mathcal \Lc_\F \Lc_{J_\F},1\}$,  $\epsLoc = \frac{\varepsilon \mu^{3/2}}{4M}$, where $\varepsilon\in (0,1)$.
        By combining \cref{thm:grad_dsct_conv} and \cref{lemma:descent_principle}, we are able to establish the following important corollary, which shows gradient descent outputs a point close to a minimizer.
    \begin{corollary}[Getting close to a minimizer]
        \label{corr:close_to_min}
        Set $\rho =  \min\left\{\frac{\epsLoc}{19\sqrt{\frac{\beta_L}{\mu}}},\sqrt{\mu}R,R\right\}$.
        Run gradient descent for $k = \frac{\beta_L}{\mu}\log\left(\frac{4\max\{2\beta_{L},1\}L(w_0)}{\mu\rho^2}\right)$ iterations, 
        gradient descent outputs a point $\wloc$ satisfying 
        \[
        L(\wloc) \leq \frac{\mu \rho^2}{4}\min\left\{1,\frac{1}{2\beta_L}\right\},
        \]
        \[
        \|\wloc-w_\star\|_{H_{L}(w_\star)+\mu I}\leq \rho,~ \text{for some}~ w_\star\in \Wstar.
        \]
    \end{corollary}
\begin{proof}
    The first claim about $L(\wloc)$ is an immediate consequence of \Cref{thm:grad_dsct_conv}.
    For the second claim, consider the ball $B(\wloc,\rho)$.
    Observe that $B(\wloc,\rho)\subset B(w_0,2R)$, so $L$ is $\mu$-\PL~in $B(\wloc,\rho)$.
    Combining this with $L(\wloc) \leq \frac{\mu \rho^2}{4}$, \Cref{lemma:descent_principle} guarantees the existence of $w_\star\in B(\wloc,\rho)\cap\Wstar$, with 
    $\|\wloc-w_\star\|\leq \sqrt{\frac{2}{\mu}L(\wloc)}$.
    Hence Cauchy-Schwarz yields 
    \begin{align*}
        \|\wloc-w_\star\|_{H_L(w_\star)+\mu I} & \leq \sqrt{\beta_L+\mu}\|\wloc-w_\star\| \leq \sqrt{2\beta_L}\|\wloc-w_\star\|\\
        & \leq 2\sqrt{\frac{\beta_L}{\mu}L(\wloc)} \leq 2 \times \sqrt{\frac{\beta_L}{\mu}\frac{\mu \rho^2}{8{\beta_L}}}\leq \rho,
    \end{align*}
   which proves the claim.
\end{proof}


\subsection{Fast Local Convergence of Damped Newton's Method}
In this section, we show damped Newton's method with fixed stepsize exhibits fast linear convergence in an appropriate region about the minimizer $w_\star$ from \cref{corr:close_to_min}. 
Fix $\varepsilon \in (0,1)$, then the region of local convergence is given by:
\[
\Neps = \left\{w\in \R^p: \|w-w_\star\|_{H_L(w_\star)+\mu I}\leq \epsLoc\right\},
\]
where $\epsLoc = \frac{\varepsilon \mu^{3/2}}{4M}$ as above. 
Note that $\wloc \in \Neps$.

We now prove several lemmas, that are essential to the argument. 
We begin with the following elementary technical result, which shall be used repeatedly below.  
\begin{lemma}[Sandwich lemma]
\label{lemma:sandwich}
    Let $A$ be a symmetric matrix and $B$ be a symmetric positive-definite matrix.
    Suppose that $A$ and $B$ satisfy $\|A-B\|\leq \varepsilon \lambda_{\textup{min}}(B)$ where $\varepsilon \in (0,1)$.
    Then
    \[
    (1-\varepsilon)B\preceq A\preceq (1+\varepsilon) B. 
    \]
\end{lemma}
\begin{proof}
    By hypothesis, it holds that
    \[
    -\varepsilon \lambda_{\textup{min}}(B)I\preceq A-B \preceq \varepsilon\lambda_{\textup{min}}(B) I.
    \]
    So using $B\succeq \lambda_{\textup{min}}(B) I$, and adding $B$ to both sides, we reach
    \[
    (1-\varepsilon)B \preceq A\preceq (1+\varepsilon) B.
    \]
\end{proof}

The next result describes the behavior of the damped Hessian in $\Neps$.
\begin{lemma}[Damped Hessian in $\Neps$]
\label{lemma:local_hess}
Suppose that $\gamma \geq \mu$ and $\varepsilon\in (0,1)$. 
\begin{enumerate}
    \item (Positive-definiteness of damped Hessian in $\Neps$) For any $w\in \Neps$, 
    \[
    \HL(w)+\gamma I \succeq \left(1-\frac{\varepsilon}{4}\right)\gamma I.
    \]
    \item (Damped Hessians stay close in $\Neps$)
    For any $w,w' \in \Neps$,
    \[
    (1-\varepsilon)\left[\HL(w)+\gamma I\right] \preceq \HL(w')+\gamma I \preceq (1+\varepsilon) \left[\HL(w)+\gamma I\right].
    \]
\end{enumerate}
\end{lemma}
\begin{proof}
    We begin by observing that the damped Hessian at $w_\star$ satisfies
    \begin{align*}
        \HL(w_\star)+\gamma I & = G(w_\star)+\gamma I+\frac{1}{n}\sum_{i=1}^{n}\left[\F(w_\star)-y\right]_{i}H_{\mathcal F_i}(w_\star)\\
        &= G(w_\star)+\gamma I \succeq \gamma I.
    \end{align*}
    Thus, $\HL(w_\star)+\gamma I$ is positive definite. 
    Now, for any $w\in \Neps$, it follows from Lipschitzness of $\HL$ that
    \begin{align*}
        \left\|\left(\HL(w)+\gamma I\right)-\left(\HL(w_\star)+\gamma I\right)\right\|\leq \Lc_{\HL}\|w-w_\star\|\leq \frac{\Lc_{\HL}}{\sqrt{\gamma}}\|w-w_\star\|_{\HL(w_\star)+\gamma I} \leq \frac{\varepsilon \mu}{4}.
    \end{align*}
    As $\lamMin\left(\HL(w_\star)+\gamma I\right)\geq \gamma>\mu$, we may invoke \Cref{lemma:sandwich} to reach 
    \[
        \left(1-\frac{\varepsilon}{4}\right)\left[\HL(w_\star)+\gamma I\right]\preceq \HL(w)+\gamma I \preceq \left(1+\frac{\varepsilon}{4}\right)\left[\HL(w_\star)+\gamma I\right].
    \]
    This immediately yields 
    \[
    \lamMin\left(\HL(w)+\gamma I\right)\geq \left(1-\frac{\varepsilon}{4}\right)\gamma \geq \frac{3}{4}\gamma,
    \]
    which proves item 1. 
    To see the second claim, observe for any $w,w'\in \Neps$ the triangle inequality implies
    \[
    \left\|\left(\HL(w')+\gamma I\right)-\left(\HL(w)+\gamma I\right)\right\|\leq \frac{\varepsilon \mu}{2} \leq \frac{2}{3}\varepsilon\left(\frac{3}{4}\gamma\right).
    \]
    As $\lamMin\left(\HL(w)+\gamma I\right)\geq \frac{3}{4}\gamma $, it follows from \Cref{lemma:sandwich} that
    \[
    \left(1-\frac{2}{3}\varepsilon\right)\left[\HL(w)+\gamma I\right]\preceq \HL(w')+\gamma I \preceq \left(1+\frac{2}{3}\varepsilon\right)\left[\HL(w)+\gamma I\right],
    \]
    which establishes item 2. 
\end{proof}

The next result characterizes the behavior of the tangent kernel and Gauss-Newton matrix in $\Neps$.
\begin{lemma}[Tangent kernel and Gauss-Newton matrix in $\Neps$]
\label{lemma:local_gn}
    Let $\gamma \geq \mu$. Then for any $w,w'\in \Neps$, the following statements hold:
    \begin{enumerate}
        \item (Tangent kernels stay close) 
        \[
        \left(1-\frac{\varepsilon}{2}\right)K_\F(w_\star)\preceq K_\F(w) \preceq \left(1+\frac{\varepsilon}{2}\right) K_\F(w_\star)
        \]
        \item (Gauss-Newton matrices stay close)
        \[
         \left(1-\frac{\varepsilon}{2}\right)\left[G(w)+\gamma I\right]\preceq G(w_\star)+\gamma I \preceq \left(1+\frac{\varepsilon}{2}\right) \left[G(w)
        +\gamma I\right]\]
        \item (Damped Hessian is close to damped Gauss-Newton matrix) 
        \[
        (1-\varepsilon)\left[G(w)+\gamma I\right] \preceq \HL(w)+\gamma I \preceq (1+\varepsilon)\left[G(w)+\gamma I\right].
        \]
        \item (Jacobian has full row-rank) The Jacobian satisfies $\textup{rank}(J_{\F}(w)) = n$.
    \end{enumerate}
\end{lemma}
\begin{proof}
\begin{enumerate}
    \item Observe that
    \begin{align*}
        \|K_\F(w)-K_\F(w_\star)\| &= \|J_{\F}(w)J_{\F}(w)^{T}-J_{\F}(w_\star)J_{\F}(w_\star)^{T}\| \\
                  &= \left\|\left[J_{\F}(w)-J_{\F}(w_\star)\right]J_{\F}(w)^{T}+J_{\F}(w_\star)\left[J_{\F}(w)-J_{\F}(w_\star)\right]^{T}\right\| \\
                  &\leq 2 \Lc_\F \Lc_{J_\F}\|w-w_\star\| \leq \frac{2 \Lc_\F \Lc_{J_\F}}{\sqrt{\gamma}}\|w-w_\star\|_{H_L(w_\star)+\gamma I} \leq \frac{\varepsilon\mu^{3/2}}{\sqrt{\gamma}} \leq \frac{\varepsilon}{2} \mu,
    \end{align*}
    where in the first inequality we applied the fundamental theorem of calculus to reach 
    \[
    \|J_{\F}(w)-J_{\F}(w_\star)\|\leq \Lc_{J_{\F}}\|w-w_\star\|.
    \]
    Hence the claim follows from \Cref{lemma:sandwich}.
    \item By an analogous argument to item 1, we find
    \[
    \left\|\left(G(w)+\gamma I\right)-\left(G(w_\star)+\gamma I\right)\right\| \leq \frac{\varepsilon}{2}\mu,
    \]
    so the result again follows from \Cref{lemma:sandwich}.
    \item First observe $\HL(w_\star)+\gamma I = G(w_\star)+\gamma I$. Hence the proof of \Cref{lemma:local_hess} implies,
    \[
    \left(1-\frac{\varepsilon}{4}\right)\left[G(w_\star)
        +\gamma I\right]\preceq \HL(w)+\gamma I\preceq \left(1+\frac{\varepsilon}{4}\right)\left[G(w_\star)
        +\gamma I\right].
    \]
    Hence the claim now follows from combining the last display with item 2. 
    \item This last claim follows immediately from item 1, as for any $w\in \Neps$,
    \[
    \sigma_{n}\left(J_{\F}(w)\right) = \sqrt{\lamMin(K_\F(w))}\geq \sqrt{\left(1-\frac{\varepsilon}{2}\right)\mu}>0.
    \]   
    Here the last inequality uses $\lamMin(K_\F(w_\star))\geq \mu$, which follows as $w_\star\in B(w_0,2R)$.
\end{enumerate}

\end{proof}

The next lemma is essential to proving convergence. It shows in $\Neps$ that $L(w)$ is uniformly smooth with respect to the damped Hessian, with nice smoothness constant $(1+\varepsilon)$. 
Moreover, it establishes that the loss is uniformly \PL with respect to the damped Hessian in $\Neps$. 
\begin{lemma}[Preconditioned smoothness and \PL]
\label{lemma:local-sm-pl}
    Suppose $\gamma \geq \mu$. Then 
    for any $w,w',w''\in \Neps$, the following statements hold:
    \begin{enumerate}
        \item $L(w'')\leq L(w')+\langle \nabla L(w'),w''-w'\rangle +\frac{1+\varepsilon}{2}\|w''-w'\|_{H_L(w)+\gamma I}^2$.
        \item $\frac{\|\nabla L(w)\|_{(\HL(w)+\gamma I)^{-1}}^2}{2}\geq \frac{1}{1+\varepsilon}\frac{1}{\left(1+\gamma/\mu\right)}L(w)$.
    \end{enumerate}
\end{lemma}

\begin{proof}
    \begin{enumerate}
        \item By Taylor's theorem
        \begin{align*}
        L(w'') = L(w')+\langle \nabla L(w'),w''-w'\rangle+\int_{0}^{1}(1-t)\|w''-w'\|_{\HL(w'+t(w''-w'))}^2 dt
        \end{align*}
        Note $w'+t(w''-w')\in \Neps $ as $\Neps$ is convex.
        Thus we have,
        \begin{align*}
            L(w'') & \leq L(w')+\langle \nabla L(w'),w''-w'\rangle+\int_{0}^{1}(1-t)\|w''-w'\|_{\HL(w'+t(w''-w'))+\gamma I}^2dt \\
            & \leq L(w')+\langle \nabla L(w'),w''-w'\rangle+\int_{0}^{1}(1-t)(1+\varepsilon)\|w''-w'\|_{\HL(w)+\gamma I}^2dt \\
            & = L(w')+\langle \nabla L(w'),w''-w'\rangle+\frac{(1+\varepsilon)}{2}\|w''-w'\|_{\HL(w)+\gamma I}^2.  
        \end{align*}
    
        \item Observe that
        \begin{align*}
            \frac{\|\nabla L(w)\|_{(\HL(w)+\gamma I)^{-1}}^2}{2} = \frac{1}{2}(\F(w)-y)^{T}\left[J_{\F}(w)\left(\HL(w)+\gamma I\right)^{-1}J_{\F}(w)^{T}\right](\F(w)-y).
        \end{align*}
        Now,
        \begin{align*}
            J_{\F}(w)\left(\HL(w)+\gamma I\right)^{-1}J_{\F}(w)^{T} & \succeq \frac{1}{(1+\varepsilon)}J_{\F}(w)\left(G(w)+\gamma I\right)^{-1}J_{\F}(w)^{T}\\ 
             &= \frac{1}{(1+\varepsilon)}J_{\F}(w)\left(J_{\F}(w)^{T}J_{\F}(w)+\gamma I\right)^{-1}J_{\F}(w)^{T}\\ 
        \end{align*}
        \Cref{lemma:local_gn} guarantees $J_{\F}(w)$ has full row-rank, so the SVD yields
        \[
        J_{\F}(w)\left(J_{\F}(w)^{T}J_{\F}(w)+\gamma I\right)^{-1}J_{\F}(w)^{T} = U\Sigma^2(\Sigma^2+\gamma I)^{-1}U^{T}\succeq \frac{\mu}{\mu+\gamma} I.
        \]
        Hence
        \[
          \frac{\|\nabla L(w)\|_{(\HL(w)+\gamma I)^{-1}}^2}{2}\geq \frac{\mu}{(1+\varepsilon)(\mu+\gamma)}\frac{1}{2}\|\F(w)-y\|^2 = \frac{\mu}{(1+\varepsilon)(\mu+\gamma)}L(w).
        \]
    \end{enumerate}
\end{proof}

\begin{lemma}[Local preconditioned-descent]
\label{lemma:local_descent}
    Run Phase II of \cref{alg-GDND} with $\eta_{\textup{DN}} = (1+\varepsilon)^{-1}$ and $\gamma = \mu$. 
    Suppose that $\tilde w_{k}, \tilde w_{k+1}\in \Neps$, then
    \[
     L(\tilde w_{k+1})\leq \left(1-\frac{1}{2(1+\varepsilon)^2}\right)L(\tilde w_k).
    \]
\end{lemma}
\begin{proof}
    As $\tilde w_k, \tilde w_{k+1}\in \Neps$, item 1 of \Cref{lemma:local-sm-pl} yields
    \[
    L(\tilde w_{k+1})\leq L(\tilde w_k)-\frac{\|\nabla L(\tilde w_k)\|^2_{(\HL(\tilde w_k)+\mu I)^{-1}}}{2(1+\varepsilon)}.
    \]
    Combining the last display with the preconditioned \PL condition, 
    we conclude
    \[
    L(\tilde w_{k+1})\leq \left(1-\frac{1}{2(1+\varepsilon)^2}\right)L(\tilde w_k).
    \]
\end{proof}

The following lemma describes how far an iterate moves after one-step of Phase II of \cref{alg-GDND}.
\begin{lemma}[1-step evolution]
\label{lemma:one_step_evol}
    Run Phase II of \cref{alg-GDND} with $\eta_{\textup{DN}} = (1+\varepsilon)^{-1}$ and $\gamma \geq \mu$.
    Suppose $\tilde w_k \in \N_{\frac{\varepsilon}{3}}(w_\star)$, then $\tilde w_{k+1}\in \Neps$.
\end{lemma}

\begin{proof}
    Let $P = H_L(\tilde w_k)+\gamma I$.
    We begin by observing that 
    \begin{align*}
        \|\tilde w_{k+1}-w_\star\|_{\HL(w_\star)+\mu I}\leq \sqrt{1+\varepsilon}\|\tilde w_{k+1}-w_\star\|_{P}.
    \end{align*}
    Now,
    \begin{align*}
        \|\tilde w_{k+1}-w_\star\|_P & = \frac{1}{1+\varepsilon}\|\nabla L(\tilde w_{k})-\nabla L(w_\star)-(1+\varepsilon)P(w_\star-\tilde w_k)\|_{P^{-1}} \\ &=
        \frac{1}{1+\varepsilon}\left\|\int_{0}^{1}\left[\nabla^2 L(w_\star+t(w_k-w_\star))-(1+\varepsilon)P\right]dt(w_\star-\tilde w_k)\right\|_{P^{-1}} \\
        & =   \frac{1}{1+\varepsilon}\left\|\int_{0}^{1}\left[P^{-1/2}\nabla^2 L(w_\star+t(w_k-w_\star))P^{-1/2}-(1+\varepsilon)I\right]dtP^{1/2}(w_\star-\tilde w_k)\right\|\\
        &\leq \frac{1}{1+\varepsilon}\int_{0}^{1}\left\|P^{-1/2}\nabla^2 L(w_\star+t(w_k-w_\star))P^{-1/2}-(1+\varepsilon)I\right\|dt\|\tilde w_k-w_\star\|_{P}
    \end{align*}
    We now analyze the matrix $P^{-1/2}\nabla^2 L(w_\star+t(w_k-w_\star))P^{-1/2}$. 
    Observe that
    \begin{align*}
        & P^{-1/2}\nabla^2 L(w_\star+t(w_k-w_\star))P^{-1/2} = P^{-1/2}(\nabla^2 L(w_\star+t(w_k-w_\star))+\gamma I-\gamma I)P^{-1/2} \\
        & = P^{-1/2}(\nabla^2 L(w_\star+t(w_k-w_\star))+\gamma I)P^{-1/2}-\gamma P^{-1} \succeq (1-\varepsilon)I-\gamma P^{-1} \succeq -\varepsilon I.
    \end{align*}
    Moreover,
    \[
    P^{-1/2}\nabla^2 L(w_\star+t(w_k-w_\star))P^{-1/2}\preceq P^{-1/2}(\nabla^2 L(w_\star+t(w_k-w_\star))+\gamma I)P^{-1/2}\preceq (1+\varepsilon)I.
    \]
    Hence, 
    \[0 \preceq (1+\varepsilon)I-P^{-1/2}\nabla^2 L(w_\star+t(w_k-w_\star))P^{-1/2}\preceq (1+2\varepsilon)I,\] 
    and so
    \[
    \|\tilde w^{k+1}-w_\star\|_P\leq \frac{1+2\varepsilon}{1+\varepsilon}\|\tilde w_k-w_\star\|_{P}.
    \]
    Thus,
    \[
    \|\tilde w^{k+1}-w_\star\|_{\HL(w_\star)+\mu I}\leq \frac{1+2\varepsilon}{\sqrt{1+\varepsilon}}\|\tilde w_k-w_\star\|_P \leq(1+2\varepsilon)\|\tilde w_k-w_\star\|_{\HL(w_\star)+\mu I}\leq \epsLoc.
    \]
\end{proof}

The following lemma is key to establishing fast local convergence; it shows that the iterates produced by damped Newton's method remain in $\Neps$, the region of local convergence. 
\begin{lemma}[Staying in $\Neps$]
\label{lemma:trapped}
    Suppose that $\wloc \in \mathcal N_{\rho}(w_\star)$, where $\rho = \frac{\epsLoc}{19\sqrt{\beta_L/\mu}}$.
    Run Phase II of \cref{alg-GDND} with $\gamma = \mu$ and $\eta = (1+\varepsilon)^{-1}$, then $\tilde w_{k+1} \in \Neps$ for all $k\geq 1$.
\end{lemma}
\begin{proof}
  In the argument that follows $\kappa_P = 2(1+\varepsilon)^2$.
  The proof is via induction. 
  Observe that if $\wloc \in \mathcal N_{\varrho}(w_\star)$ then by \Cref{lemma:one_step_evol}, $\tilde w_1 \in \Neps$.  
  Now assume $\tilde w_j \in \Neps$ for $j = 2,\dots, k$. 
  We shall show $\tilde w_{k+1}\in \Neps$.
  To this end, observe that
  \begin{align*}
  \|\tilde w_{k+1}-w_\star\|_{\HL(w_\star)+\mu I} & \leq \|\wloc-w_\star\|_{\HL(w_\star)+\mu I}+\frac{1}{1+\varepsilon}\sum_{j=1}^{k}\|\nabla L(w_j)\|_{\left(\HL(w_\star)+\mu I\right)^{-1}} \\
  \end{align*}
  Now,
  \begin{align*}
      \|\nabla L(w_j)\|_{\left(\HL(w_\star)+\mu I\right)^{-1}} &\leq \frac{1}{\sqrt{\mu}}\|\nabla L(w_j)\|_2 \leq \sqrt{\frac{2\beta_L}{\mu}L(w_j) }\\
      &\leq \sqrt{\frac{2\beta_L}{\mu}}\left(1-\frac{1}{\kappa_P}\right)^{j/2}\sqrt{L(\wloc)},
  \end{align*}
  Here the second inequality follows from $\|\nabla L(w)\| \leq \sqrt{2\beta_L L(w)}$, and the last inequality follows from \Cref{lemma:local_descent}, which is applicable as $\tilde w_{0},\dots,\tilde w_k \in \Neps$. 
  Thus,
  \begin{align*}
  \|\tilde w_{k+1}-w_\star\|_{\HL(w_\star)+\mu I} &\leq \rho+\sqrt{\frac{2\beta_L}{\mu}}\sum_{j=1}^{k}\left(1-\frac{1}{\kappa_P}\right)^{j/2}\sqrt{L(\tilde w_0)} \\
  &\leq \rho+\sqrt{\frac{(1+\varepsilon)\beta_L}{2\mu}}\|\wloc - w_\star\|_{\HL(w_\star)+\mu I}\sum_{j=1}^{k}\left(1-\frac{1}{\kappa_P}\right)^{j/2}\\
  &\leq \left(1+\sqrt{\frac{\beta_L}{\mu}}\sum_{j=0}^{\infty}\left(1-\frac{1}{\kappa_P}\right)^{j/2}\right)\rho\\
  & = \left(1+\frac{\sqrt{\beta_L/\mu}}{{1-\sqrt{1-\frac{1}{\kappa_P}}}}\right)\rho\leq \epsLoc.
  \end{align*}
  Here, in the second inequality we have used $L(\tilde w_0)\leq 2(1+\varepsilon)\|\wloc - w_\star\|^2_{\HL(w_\star)+\mu I}$, which is an immediate consequence of \cref{lemma:local-sm-pl}.
  Hence, $\tilde w_{k+1}\in \Neps$, and the desired claim follows by induction. 
\end{proof}

\begin{theorem}[Fast-local convergence of Damped Newton]
\label{thm:dn_fast_loc}
    Let $\wloc$ be as in \cref{corr:close_to_min}. 
    Consider the iteration 
    \[
    \tilde w_{k+1} = \tilde w_k-\frac{1}{1+\varepsilon}(\HL(\tilde w_k)+\mu I)^{-1}\nabla L(\tilde w_k),\quad \text{where}~\tilde w_0 = \wloc.\] 
    Then, after $k$ iterations, the loss satisfies
        \[
        L(\tilde w_k) \leq \left(1-\frac{1}{2(1+\varepsilon)^2}\right)^{k}L(\wloc).
        \]
        Thus after $k = \mathcal O\left(\log\left(\frac{1}{\epsilon}\right)\right)$ iterations
        \[
        L(\tilde w_k)\leq \epsilon.
        \]
    \begin{proof}
        \cref{lemma:trapped} ensure that $\tilde w^{k} \in \Neps$ for all $k$.
         Thus, we can apply item $1$ of \Cref{lemma:local-sm-pl} and the definition of $\tilde w^{k+1}$, to reach
         \[
         L(\tilde w_{k+1})\leq L(\tilde w_{k})-\frac{1}{2(1+\varepsilon)}\|\nabla L(\tilde w_k)\|_{P^{-1}}^2. 
         \]
         Now, using item $2$ of \Cref{lemma:local-sm-pl} and recursing yields  
         \[
         L(\tilde w_{k+1})\leq \left(1-\frac{1}{2(1+\varepsilon)^2}\right)L(\tilde w_k)\leq \left(1-\frac{1}{2(1+\varepsilon)^2}\right)^{k+1}L(\wloc).
         \]
         The remaining portion of the theorem now follows via a routine calculation.
    \end{proof}
\end{theorem}

\subsection{Formal Convergence of \cref{alg-GDND}}
\label{subsec:GDND_conv}
Here, we state and prove the formal convergence result for \cref{alg-GDND}.
\begin{theorem}
\label{thm:GDND}
    Suppose that \cref{assp:interpolation} and \cref{assp:loss_reg} hold, and that the loss is $\mu$-\PL in $B(w_0,2R)$, where $R = \frac{2\sqrt{2\beta_L L(w_0)}}{\mu}$.
    Let $\epsLoc$ and $\rho$ be as in \cref{corr:close_to_min}, and set $\varepsilon = 1/6$ in the definition of $\epsLoc$. 
    Run \cref{alg-GDND} with parameters: $\eta_{\textup{GD}} = 1/\beta_L, K_{\textup{GD}} = \frac{\beta_L}{\mu}\log\left(\frac{4\max\{2\beta_{L},1\}L(w_0)}{\mu\rho^2}\right), \eta_{\textup{DN}} = 5/6, \gamma = \mu$ and $K_{\textup{DN}}\geq 1$.
    Then Phase II of \cref{alg-GDND} satisfies
    \[
    L(\tilde w_{k})\leq \left(\frac{2}{3}\right)^{k}L(w_{K_{\textup{GD}}}).
    \]
    Hence after $K_{\textup{DN}} \geq 3\log\left(\frac{L(w_{K_{\textup{GD}}})}{\epsilon}\right)$ iterations, Phase II of \cref{alg-GDND} outputs a point satisfying
    \[
     L(\tilde w_{K_{\textup{DN}}})\leq \epsilon.
    \]
\end{theorem}

\begin{proof}
    By assumption the conditions of \cref{corr:close_to_min} are met, therefore $w_{K_{\textup{GD}}}$ satisfies $\|w_{K_{\textup{GD}}}-w_\star\|_{H_L(w_\star)+\mu I}\leq \rho$, for some $w_\star \in \Wstar$.
    Hence, we may invoke \cref{thm:dn_fast_loc} to conclude the desired result. 
\end{proof}




\end{document}